\documentclass[preprint,12pt]{elsarticle}

\usepackage[linesnumbered,ruled]{algorithm2e}
\usepackage{amsmath}
\usepackage{amsthm}
\usepackage{amssymb}
\usepackage{amsfonts}
\usepackage{thmtools} 
\usepackage{thm-restate}
\usepackage{stmaryrd}
\usepackage{paralist}
\usepackage[shortlabels]{enumitem}
\usepackage{xcolor}
\usepackage{soul}
\usepackage[aboveskip = 6pt]{caption}
\usepackage{comment}
\usepackage{tabularx}
\usepackage{multirow}
\usepackage{booktabs}
\usepackage{longtable}
\usepackage{graphicx}
\usepackage{subfig}
\usepackage[figuresright]{rotating}
\usepackage[colorlinks, linkcolor=blue, citecolor=blue]{hyperref}
\usepackage{latexsym} 
\usepackage{appendix}
\usepackage[colorinlistoftodos,prependcaption,textsize=tiny]{todonotes}
\graphicspath{{../Figures/}{../../Figures/}{../Figures/experiments/}{../Figures/variable-order/}{../Figures/OBDDCs/}}
\theoremstyle{plain}%default
\newtheorem{theorem}{Theorem}
\newtheorem{proposition}{Proposition}

\theoremstyle{definition}
\newtheorem{definition}{Definition}
\newtheorem{example}{Example}
\theoremstyle{remark}

\newtheorem{claim}{Claim}

\newcommand{\textquoted}[1]{\textquotedblleft #1\textquotedblright{}}
\newcommand{\rytodo}[1]{\todo[color=cyan]{Roland: #1}}

\newcommand{\Floor}[1]{\lfloor{#1}\rfloor}
\newcommand{\Ceil}[1]{\lceil{#1}\rceil}
\newcommand{\Angle}[1]{\langle #1 \rangle}

\newcommand{\NOT}{\neg}
\newcommand{\AND}{\wedge}
\newcommand{\OR}{\vee}
\newcommand{\EQU}{\leftrightarrow}

\newcommand{\D}{\mathcal{D}}

\newcommand{\OBDDC}[2]{OBDD$[\AND_{#1}]_{#2}$}

\newcommand{\CDD}{CDD}

\newcommand{\CCDD}{CCDD}

\newcommand{\NNF}{\ensuremath{\mathsf{NNF}}}
\newcommand{\DNNF}{\ensuremath{\mathsf{DNNF}}}
\newcommand{\dDNNF}{\ensuremath{\mathsf{d\text{-}DNNF}}}
\newcommand{\DecDNNF}{Decision-\ensuremath{\mathsf{DNNF}}}

\newcommand{\CNF}{\ensuremath{\mathsf{CNF}}}

\newcommand{\LOBDD}[1]{\ensuremath{\mathsf{OBDD_{#1}}}}

\newcommand{\LSDD}[1]{\ensuremath{\mathsf{SDD_{#1}}}}

\newcommand{\LOBDDC}[2]{\ensuremath{\mathsf{OBDD[\AND_{#1}]_{#2}}}}

\newcommand{\Rzero}{\ensuremath{\mathsf{CDD}}}
\newcommand{\Ctwo}{\ensuremath{\mathsf{CCDD}}}
\newcommand{\CtwoK}[1]{\ensuremath{\mathsf{CCDD}[k_{#1}]}}

\newcommand{\satisfying}[1]{\ensuremath{\mathsf{sol}(#1)}}

\newcommand{\dfour}{D4}
\newcommand{\PreLite}{\ensuremath{\mathsf{PreLite}}}
\newcommand{\Panini}{\ensuremath{\mathsf{Panini}}}
\newcommand{\ExactMC}{\ensuremath{\mathsf{ExactMC}}}
\newcommand{\ExactSamp}{\ensuremath{\mathsf{ExactUS}}}
\DontPrintSemicolon
\SetKwBlock{KwFunc}{function}{end}

\newcommand{\ConstructCore}{\textsc{ConstructCore}}

\newcommand{\DetectLitEqu}{\textsc{DetectLitEqu}}

\newcommand{\Decompose}{\textsc{Decompose}}
\newcommand{\FuncPickGoodVar}{\textsc{PickGoodVar}}
\newcommand{\ReduceKDepth}{\textsc{RedKDepth}}
\newcommand{\DecodeLitEqu}{\text{DecodeLitEqu}}
\newcommand{\Kernelizable}{\textsc{ShouldKernelize}}
\newcommand{\EqLits}{\ensuremath{\mathsf{eqLits}}}
\newcommand{\SearchCounter}{\ensuremath{\mathsf{SearchCounter}}}
\newcommand{\Cache}{\ensuremath{\mathit{Cache}}}

\newcommand{\oCT}{\ensuremath{\mathit{CT}}}

\newcommand{\CO}{\ensuremath{\mathbf{CO}}}
\newcommand{\VA}{\ensuremath{\mathbf{VA}}}
\newcommand{\CE}{\ensuremath{\mathbf{CE}}}
\newcommand{\IM}{\ensuremath{\mathbf{IM}}}
\newcommand{\EQ}{\ensuremath{\mathbf{EQ}}}
\newcommand{\SE}{\ensuremath{\mathbf{SE}}}
\newcommand{\CT}{\ensuremath{\mathbf{CT}}}
\newcommand{\ME}{\ensuremath{\mathbf{ME}}}
\newcommand{\CD}{\ensuremath{\mathbf{CD}}}
\newcommand{\SFO}{\ensuremath{\mathbf{SFO}}}
\newcommand{\FO}{\ensuremath{\mathbf{FO}}}
\newcommand{\ABC}{\ensuremath{\mathbf{\wedge}\mathbf{BC}}}
\newcommand{\AC}{\ensuremath{\mathbf{\wedge}\mathbf{C}}}
\newcommand{\OBC}{\ensuremath{\mathbf{\vee}\mathbf{BC}}}
\newcommand{\OC}{\ensuremath{\mathbf{\vee}\mathbf{C}}}
\newcommand{\NC}{\ensuremath{\mathbf{\neg}\mathbf{C}}}

\begin{document}

\begin{frontmatter}



% \title{Constrained Conjunction \& Decision Diagrams: A Tractable Representation for Model Counting and Uniform Sampling\protect\thanks{This is an extended version of the paper entitled ``The Power of Literal Equivalence in Model Counting'' published in the proceedings of AAAI-21 (3851--3859). The author list has been sorted alphabetically by last name; this should not be used to determine the extent of authors' contributions.}}
%\title{Constrained Conjunction \& Decision Diagrams: A Tractable Representation for Model Counting and Uniform Sampling\tnoteref{t1}}
\title{\Ctwo{}: A Tractable Representation for Model Counting and Uniform Sampling\tnoteref{t1}}
\tnotetext[t1]{This is an extended version of the paper entitled ``The Power of Literal Equivalence in Model Counting'' published in the proceedings of AAAI-21 (3851--3859). The author list has been sorted alphabetically by last name; this should not be used to determine the extent of authors' contributions.}


%\author[label1]{Yong Lai}
%\author[label2]{Kuldeep S. Meel}
%\author[label2]{Roland H. C. Yap}

\author[1]{Yong Lai\corref{cor1}}
\author[2]{Kuldeep S. Meel}
\author[2]{Roland H. C. Yap}

% \affiliation[label1]{organization={Key Laboratory of Symbolic Computation and Knowledge Engineering of Ministry of Education, Jilin University},%Department and Organization
% %%            addressline={}, 
%             city={Changchun},
%             postcode={130012}, 
% %%            state={},
%             country={China}}

\address[1]{Key Laboratory of Symbolic Computation and Knowledge Engineering of Ministry of Education, Jilin University, Changchun, 130012, China}

% \affiliation[label2]{organization={School of Computing, National University of Singapore},
% 	%%             addressline={},
% 	%%             city={},
%              postcode={119077},
% 	%%             state={},
% 	             country={Singapore}
% }

\address[2]{School of Computing, National University of Singapore, 119077, Singapore}

%\address[3]{School of Computing, National University of Singapore, 119077, Singapore}

\cortext[cor1]{Corresponding author}

\begin{abstract}
Knowledge compilation concerns with the compilation of representation languages to target languages supporting a wide range of tractable operations arising from diverse areas of computer science. 
Tractable target compilation languages are usually achieved by restrictions on the internal nodes ($\land$ or $\lor$) of the \NNF{}. 
In this paper, we propose a new representation language \Ctwo{}, which introduces new restrictions on conjunction nodes to capture equivalent literals. We show that \Ctwo{} supports two key queries, model counting and uniform samping, in polytime. 
We present algorithms and a compiler to compile propositional formulas expressed in \CNF{} into \Ctwo{}. 
Experiments over a large set of benchmarks show that
our compilation times are better with smaller representations than
state-of-art \DecDNNF{}, \LSDD{} and \LOBDDC{}{} compilers. 
We apply our techniques to model counting and uniform sampling, and develop model counter and uniform sampler on \CNF{}.
Our empirical evaluation demonstrates the following significant improvements: our model counter can solve 885 instances while the prior state of the art solved only 843 instances, representing an improvement of 43 instances; and our uniform sampler can solve 780 instances while the prior state of the art solved only 648 instances, representing an improvement of 132 instances. 
\end{abstract}
\begin{comment}
\begin{graphicalabstract}
\includegraphics[width = 0.5\linewidth]{figures/compilation-with-pmc-time.pdf}
\includegraphics[width = 0.5\linewidth]{figures/compilation-with-pmc-sizes-E5.pdf}
\includegraphics[width = 0.5\linewidth]{figures/counting-with-BE-time.pdf}
\includegraphics[width = 0.5\linewidth]{figures/sampling-with-pmc-time-ExactUS.pdf}
\end{graphicalabstract}

\begin{highlights}
\item We propose a new representation language \Ctwo{}, which supports two key queries, model counting and uniform samping, in polytime. 
\item Our model counter based on \Ctwo{} can solve 885 instances while the prior state of the art solved only 843 instances, representing an improvement of 43 instances
\item Our uniform sampler based on \Ctwo{} can solve 780 instances while the prior state of the art solved only 648 instances, representing an improvement of 132 instances. 
\end{highlights}
\end{comment}


\end{frontmatter}


\section{Introduction}
\label{sec:intro}

Propositional reasoning plays a key role in diverse areas ranging from artificial intelligence, computational biology, verification, and the like. The computational intractability of the basic queries such as satisfiability, clausal entailment, and model counting for propositional reasoning provided an impetus to the emergence of the knowledge compilation (KC) approach \cite{Selman:Kautz:96,Darwiche:Marquis:02,Cadoli:Donini:97}. Knowledge compilation concerns with the compilation of propositional theory into target languages that support a wide range of queries including satisfiability, model counting, uniform sampling in polynomial time. Accordingly, KC-based techniques form the core of several inference techniques in the context of probabilistic databases \cite{VdB:Suciu:17}, probabilistic programming \cite{Fierens:etal:15}, tractable learning \cite{Kisa:etal:14}, and for synthesis and verification of hardware and software systems \cite{Fried:etal:16,book:Clarke:etal:00}.

A target language is measured across three dimensions~\cite{Darwiche:Marquis:02,Huang:Darwiche:07,Marquis:15}: (1) succinctness of the target language; (2) supported operations in polytime by the target language; 
and (3) runtime efficiency of compilation process from representation to target language. 
%\todo{Yong: Kuldeep, can you help to improve this paragraph}
%The seminal work of Darwiche and Marquis \shortcite{Darwiche:Marquis:02} 
%established the KC map to catalog different approaches across dimensions
%of the target compilation language: (1) succinctness of the target language; (2) supported queries in polytime by the target language; 
%and (3) runtime complexity of compilation process from representation to target language. 
The design of target compilation languages typically focuses on propositional formulas in negation normal form where the internal nodes are either conjunction $(\wedge)$ or disjunction $(\vee)$, and the leaf nodes are $\top$ ($true$), $\bot$ ($\mathit{false}$), $x$, $\neg x$ for variable $x$. To achieve tractability, we often put restrictions on the internal nodes with respect to their children. Two of the most widely used restrictions to achieve tractability are decomposability and determinism \cite{Darwiche:01a,Darwiche:01b}.
%(1) determinism, and (2) decomposability.

% wherein determinism ensures the solutions space of all the children of a disjunction node to have disjoint solution space while decomposability ensures that the children of a conjunction node to be defined over a disjoint subset of variables \cite{Darwiche:01a,Darwiche:01b}.    

Due to the ubiquity of \CNF{} as representation language, we are often interested in compilation methods from \CNF{} to the desired 
target compilation language. 
The restrictions to achieve tractability are designed while keeping the the runtime complexity of the compilation  in consideration. 
In practice, we often use {\em decision} nodes to enforce determinism. In contrast, the decomposability can be enforced by a simple clustering of CNF clauses such that clauses in distinct clusters do not share variables, and thereafter a conjunction node with children corresponding to each of the clusters can be constructed. Given the intractability of satisfiability on \CNF{}, syntactic structure-based restrictions ensure the creation of a node can be achieved in polynomial time; the need for exponentially many nodes for most interesting target languages still leads to exponential time compilation algorithms.

While the KC map studies a diverse set of operations and properties, we focus our attention on  model counting (CT) 
and uniform sampling (US) queries owing to their widespread usage in diverse areas ranging from probabilistic inference, reliability of networks, to hardware and software model checking, etc. 
\DecDNNF{} \cite{Oztok:Darwiche:14}, an influential target language, has been shown to support tractable model counting and uniform sampling. 
Actually, it was observed by Huang and Darwiche~\cite{Huang:Darwiche:07} that the trace of a search-based exact model counter corresponds to \DecDNNF{}.
Furthermore, Sharma et al. \cite{KUS} showed that a scalable uniform sampler was engineered based on the scalable knowledge compiler \dfour{} \cite{D4} on \DecDNNF{}.
The starting point of our work is to investigate the following natural question: {\em Can we design efficient techniques on model counting and uniform sampling based on a generalization of \DecDNNF{}}?

The primary contribution of this paper is an affirmative answer to the above question. 
As a first step, we observe that the widely employed restrictions, in the context of knowledge compilation, on the internal nodes, decomposability, and determinism, are not expressive enough to capture literal equivalences.
Indeed, pre-/in-processing techniques are an important step in modern SAT solvers~\cite{book-chapter:Marques-Silva:etal:09}.
%\rytodo{is that the correct handbook citation?} 
We then first propose a generalization of \DecDNNF{}, called {\Ctwo}, to capture literal equivalence, and show that {\Ctwo} supports model counting and uniform sampling in polynomial time. 
Guided by our motivation, we now design a knowledge compiler, called \Panini{}, to compile \CNF{} formulas into {\Ctwo}, and apply it to model counting and uniform sampling. 

To empirically measure the effectiveness of {\Ctwo}, we perform an extensive experimental evaluation over a comprehensive set of benchmarks and conduct performance comparison of our tools vis-a-vis the state of the art knowledge compilers, model counters, and uniform samplers, c2d~\cite{c2d}, Dsharp~\cite{Dsharp}, miniC2D~\cite{Oztok:Darwiche:15}, BDDC~\cite{Lai:etal:17}, D4~\cite{D4}, ADDMC \cite{ADDMC}, Ganak \cite{SRSM19}, SPUR \cite{SPUR}, and KUS~\cite{KUS}. 
Our empirical evaluation over a large set of benchmarks show that
our compilation times are better with smaller representations than
state-of-art \DecDNNF{}, \LSDD{}, and \LOBDDC{}{} compilers. 
Among the prior state of the art model counters and uniform samplers are 843 (Ganak) and 648 (SPUR), our counter and sampler solve 886 and 780, representing a significant improvement of 43 and 132 instances, respectively. 
Since the developments in KC techniques have demonstrated the significance of engineering improvements, we believe that the significant performance improvements of our tools open up directions of future research in the improvement of decision heuristics, caching schemes, and the like for compilers, counters, samplers based on {\Ctwo}.

The rest of the paper is organized as follows.  We present notations, preliminaries, and related work in Sections \ref{sec:prelims}--\ref{sec:related}. 
We introduce {\Ctwo} in Section~\ref{sec:CDD} to capture literal equivalence, and tractable algorithms for model counting and uniform sampling in Section~\ref{sec:CT-US}. 
In Section~\ref{sec:tools}, we present our tools for knowledge compilation, model counting, and uniform sampling. 
Next, we present detailed empirical evaluation in Section~\ref{sec:experiments}. 
Finally, we discuss the other tractable operations on {\Ctwo} in Section~\ref{sec:tract} and conclude in Section~\ref{sec:conclusion}.

\section{Notations and Background}\label{sec:prelims}

In a formula or the representations discussed,
$x$ denotes a propositional variable, and
literal $l$ is a variable $x$ or its negation $\neg x$, where $var(l)$
denotes the variable.
$\mathit{PV} = \{x_0, x_1, \ldots, x_n, \ldots\}$ denotes a
set of propositional variables. 
A formula is constructed from constants $\mathit{true}$, $\mathit{false}$ and propositional variables using negation operator $\lnot$, conjunction operator $\land$, disjunction operator $
\lor$, and equality operator $\EQU$.
A clause $C$ (resp. term $T$) is a set of literals representing their disjunction (resp. conjunction).
A formula in conjunctive normal form (\CNF{}) is a set of clauses representing their conjunction.  
Given a formula $\varphi$, a variable $x$, and a constant $b$, a substitution $\varphi[x \mapsto b]$ is a transformed formula by replacing $x$ by $b$ in $\varphi$.
An assignment $\omega$ over a variable set $X$ is a mapping from $X$ to $\{true, \mathit{false}\}$. 
Given a literal $l$, we denote $\omega(l)$ by $\{var(l) = \mathit{true}\}$ if $l$ is positive and $\{var(l) = \mathit{false}\}$ otherwise.
The set of all assignments over $X$ is denoted by $2^X$. 
A model of $\varphi$ is an assignment over $\mathit{Vars}(\varphi)$ that satisfies $\varphi$; that is, the substitution of $\varphi$ on the model equals to $\mathit{true}$.  Let $\satisfying{\varphi} \subseteq 2^{X}$ represent the set of models of $\varphi$, and $\varphi \models \psi$ iff $\satisfying{\varphi} \subseteq \satisfying{\psi}$. Given a formula $\varphi$, the problem of model counting is to compute $|\satisfying{\varphi}|$, and the problem of uniform sampling is to generate a random model in $\satisfying{\varphi}$ with the same probability $\frac{1}{|\satisfying{\varphi}|}$.

We focus on subsets of Negation Normal Form (\NNF{}) where the internal nodes are labeled with disjunction ($\vee$) or conjunction ($\wedge$) while the leaf nodes are labeled with  $\bot$ ($false$), $\top$ ($true$), 
or a literal.
% $x$, or $\neg x$ for variable $x$. 
For a node $v$, let $\vartheta(v)$ and $Vars(v)$ denote the formula represented by the DAG rooted at $v$, and the variables that label the descendants of $v$, respectively.  

We define the well-known decomposed conjunction~\cite{Darwiche:Marquis:02}
as follows:
\begin{definition}
	A conjunction node $v$ is called a \emph{decomposed conjunction} if its children (also known as conjuncts of $v$) do not share variables. Formally, let $w_1, \ldots, w_k$ be the children of $\AND$-node $v$, then $\mathit{Vars}(w_i) \cap \mathit{Vars}(w_j) = \emptyset $ for $i \ne j$. 
\end{definition}

If each conjunction node is decomposed, we say the formula is in \emph{Decomposable} \NNF{} (\DNNF{})~\cite{Darwiche:01a}.

\begin{definition}
	A disjunction node $v$ is called \emph{deterministic} if each two disjuncts of $v$ are logically contradictory. That is, if $w_1$, \ldots, $w_n$ are
	the children of $\OR$-node $v$, then $\vartheta(w_i) \AND \vartheta(w_j) \models false$ for $i \ne j$. 
\end{definition}

If each disjunction node of a \DNNF{} formula is deterministic, we say the formula is in deterministic \DNNF{} (\dDNNF{}), and we can perform tractable model counting on it.

\emph{Binary decision} is a practical property to impose determinism in the design of a compiler (see e.g., \dfour{} \cite{D4}), and the resulting language is called \DecDNNF{} \cite{Oztok:Darwiche:14}. 
Essentially, each decision node with one variable $x$ and two children is equivalent to a disjunction node of the form $(\NOT x \AND \varphi) \OR (x \AND \psi)$, where $\varphi$, $\psi$ represent the formulas corresponding to the children.
If each node of an \NNF{} formula is labeled with $\bot$ or $\top$, or represents a binary decision, the formula is called a Binary Decision Diagram (BDD).

%A BDD satisfies the \emph{read-once property} iff on each path from the root to a leaf, a variable appears at most once; and the result diagram is called free BDD \cite{Gergov:Meinel:94}, which is a special \dDNNF{} formula.
Given a linear ordering $\prec$, a BDD is called ordered (OBDD) \cite{Bryant:86} if each decision node $u$ with variable $x_i$ and its decision descendant $v$ with variable $x_j$ satisfy $x_i \prec x_j$. 
Darwiche \cite{Darwiche:11} generalized binary decision to sentential decision, and proposed the sentential decision diagram (SDD). 
Lai at al. \cite{Lai:etal:17} augmented OBDD with decomposed conjunction giving \OBDDC{}{}.
% Lai et al. \shortcite{JAIR17} augmented OBDD with decomposed conjunction, and the resulting diagram is called \OBDDC{}{}.
Hereafter, we will use \LOBDD{}, \LSDD{} and \LOBDDC{}{} to denote the sets of all OBDDs, SDDs and \OBDDC{}{}s, respectively.
%Note that these three languages are subset a tractable equivalence check.

\section{Related Work}
\label{sec:related}

\subsection{Knowledge Compilation}

In the context of knowledge compilation, a diverse set of operations and properties have been studied with respect to the KC map \cite{Darwiche:Marquis:02}. However, we focus our attention on  model counting (CT) and uniform sampling (US) queries owing to their widespread usage in diverse areas ranging from probabilistic inference, reliability of networks, to hardware and software model checking, etc. 

To the best of our knowledge, \DecDNNF{} is the first KC language which has been shown to support both tractable CT and US \cite{Darwiche:01b,KUS}. Natually, the subsets of \DecDNNF{}, e.g., \LOBDD{}, \LSDD{}, and \LOBDDC{}{}, also support both tractable CT and US.
This paper generalizes \DecDNNF{} to propose a new representation \Ctwo{} that also supports both tractable CT and US.
On the other hand, there are some languages, e.g., Sym-\ensuremath{\mathsf{DDG}} \cite{Bart:etal:14} and \ensuremath{\mathsf{EADT}} \cite{Koriche:etal:13}, which support tractable CT but is unknown to support tractable US or not.
For KC tools, there are many practical knowledge compilers so far, including c2d \cite{c2d}, Dsharp \cite{Dsharp}, miniC2D \cite{Oztok:Darwiche:15}, BDDC \cite{Lai:etal:17}, and D4 \cite{D4}.
We remark that many BDD packages (e.g., CUDD \cite{CUDD} and BuDDy \cite{BuDDy}) and the SDD package \cite{SDD-package} also equip the operations to transform a CNF formula into the corresponding KC languages.
In addition, some of knowledge compilers, e.g., c2d \cite{c2d}, Dsharp \cite{Dsharp}, miniC2D \cite{Oztok:Darwiche:15}, and D4 \cite{D4}, also implement the interface for CT and therefore can serve as scalable model counters.

\subsection{Model Counting}

In this paper, we focus on the design of search-based model counters. To this end, we first present the skeleton of a general search-based model counter in Algorithm~\ref{alg:Search}.\footnote{To improve readability, we slightly modified the fashion of calculating the current count to be consistent with our \ExactMC{} algorithm. $X$ is the set of variables in the original formula.} The algorithms often maintain a cache that stores the residual sub-formulas along with their corresponding model counts. The component-based decomposition, represented in line~\ref{line:alg-search-decompose}, seeks to partition the $\varphi$ into sub-formulas, referred to as components, such that each of the components is defined over a mutually disjoint set of variables. Else, we pick a variable in line \ref{line:alg-search-decision-begin} and recursively compute the exact model count. Huang and Darwiche observed that the trace of the execution of such a model counter could be viewed to correspond to a {\DecDNNF} formula. In this context, it is worth emphasizing that {\DecDNNF} supports linear time model counting, which is reflected in simple constant time computations in lines~\ref{line:alg-search-decompose-count} and~\ref{line:alg-search-decision-count} during each step of the recursions wherein every step of the recursion would correspond to a node in {\DecDNNF} capturing the trace of the execution of {\SearchCounter}. 
In this paper, we implemented a new model counter called \ExactMC{} based on a generalized framework of {\SearchCounter}. 

\begin{algorithm}[tb]
	%\footnotesize
	\caption{{\SearchCounter}($\varphi$)} \label{alg:Search}
	\lIf {$\varphi = \mathit{false}$} {\KwRet 0} \label{line:alg-search-base-case-begin}
	\lIf {$\varphi = \mathit{true}$} {\KwRet $2^{|X|}$} \label{line:alg-search-base-case-end}
	\lIf {$\Cache(\varphi) \not= nil$} {\KwRet $\Cache(\varphi)$}
	$\Psi \leftarrow \Decompose(\varphi)$\; \label{line:alg-search-decompose}
	\uIf {$|\Psi| > 1$} {
		$c \leftarrow \prod_{\psi \in \Psi}{\SearchCounter}(\psi)$\;
		\KwRet $\Cache(\varphi) \leftarrow \frac{c}{2^{(|\Psi| - 1) \cdot |X|}}$ \label{line:alg-search-decompose-count}
		%	\KwRet $\mathit{Cache}(\varphi)$
	}
	\Else{ 
		$x \leftarrow \FuncPickGoodVar(\varphi)$\; \label{line:alg-search-decision-begin}
		$c_0 \leftarrow \SearchCounter(\varphi[x \mapsto false] )$\;
		$c_1 \leftarrow \SearchCounter(\varphi[x \mapsto true] )$\; 
		
		\KwRet $\mathit{Cache}(\varphi) \leftarrow \frac{c_0 + c_1}{2}$\;\label{line:alg-search-decision-count}
		
	}
\end{algorithm}

\paragraph{Remark on Approximate Model Counting}

While this work focuses on exact model counting, it is worth remarking that there has been a 
long line of work in the design of efficient hashing-based approximate model counters that seek to provide 
$(\varepsilon,\delta)$-guarantees~\cite{Stockmeyer:83,Gomes:etal:06,ApproxMC,ApproxMC2,ApproxMC3,ApproxMC4}. 
%The key idea of hashing-based approximate model counters is to employ XOR-based hash functions to partition the solution space of the formula $\varphi$ into {\em roughly equal} small cells of solutions such that we can estimate the number of solutions of $\varphi$ by scaling the number of solutions in a randomly chosen cell by the total number of cells.

\subsection{Uniform Sampling}

Uniform sampling is closely related to model counting and knowledge compilation. 
Recently, Achlioptas et al. \cite{SPUR} brought together model counting and reservoir sampling to develop a uniform sampler called SPUR on top of sharpSAT. 
Subsequently, Sharma et al. \cite{KUS} showed that \DecDNNF{} supports tractable uniform sampling and proposed a uniform sampler called KUS using D4.
Furthermore, this paper shows that \Ctwo{} also supports tractable US,  and a scalable uniform sampler called \ExactSamp{} is developed based on a more efficient knowledge compiler \Panini{} than \dfour{}.

While this work focuses on uniform sampling, there are also many samplers that seek to achieve scalability at the cost of theoretical guarantees of uniformity.
Chakraborty et al. \cite{UniGen} introduced the first practical almost-uniform sampler, UniGen, which has been improved
to UniGen3 \cite{ApproxMC4}.
Golia et al. \cite{CMSgen} designed a sampler called CMSGen by modifying the existing state-of-the-art Conflict-Driven Clause Learning
(CDCL) SAT solver CryptoMiniSat \cite{CryptoMiniSat}. 
Although no theoretical guarantee has been provided, CMSGen performs very well in practice.

\section{Capturing Literal Equivalences by \Ctwo{}}%\yltodo{Changed}}
\label{sec:CDD}

To seek an answer to the natural question of designing a counter whose trace is a generalization of \DecDNNF{}, we first investigate appropriate generalizations of \DecDNNF{}. To this end, 
we turn to the literal equivalences, a powerful technique in SAT solving, and we design a new representation language that seeks to utilize literal equivalences. We first discuss how to capture literal equivalence from the knowledge compilation perspective, which is then manifested into a corresponding new tractable language, called {\Ctwo}. We finally show that {\Ctwo} supports linear model counting, which serves as motivation for us to design a counter whose trace corresponds to {\Ctwo}. 

\subsection{Capturing Literal Equivalences}
\label{sec:CDD:lit-equ}
% \yltodo{This subsection is used to answer Comment 2 from R1}
Given two literals $l$ and $l'$, we use $l \EQU l'$ to denote literal equivalence of $l$ and $l'$.  
%Note that literal equivalence is an equivalence relation. 
Given a set of literal equivalences $E$, let $E' = \{l \EQU l', \NOT l \EQU \NOT l'\mid l \EQU l' \in E\}$; and then we define semantic closure of $E$,  denoted by $\Ceil{E}$, as equivalence closure of $E'$. Now for every literal $l$ under $\Ceil{E}$, let $[l]$ denote the equivalence class of $l$. Given $E$, a unique equivalent representation of $E$, denoted by $\Floor{E}$ and called {\em prime literal equivalences}, is defined as follows: \\ 
\centerline{
	$\displaystyle
	\Floor{E} =   \bigcup\limits_{x \in \mathit{PV}, \min_\prec[x] = x}^{} \{ x \EQU l \mid l \in [x], l \neq x \} $
} \\
% \begin{align*}
% \Floor{E} =   \bigcup\limits_{l \in PV, \min_\prec[l] = l; l' \neq l, l' \in [l]}^{} \{ l \EQU l' \} 
% \end{align*}
where $\min_\prec[x]$ is the minimum variable appearing in $[x]$ over the lexicographic order $\prec$. 
%Note that we can always ensure that every element of $\Floor{E}$ is of the form $x \EQU l$ where $x$ is a variable and $l$ is a literal. 
It can be shown that $\Ceil{E} = \Ceil{\Floor{E}}$.  

Let $\varphi$ be a formula and let $E$ be a set of prime literal equivalences implied by $\varphi$. 
We can obtain another formula $\varphi'$ by performing a \emph{literal-substitution}: replace each $l$ (resp. $\NOT l$) in $\varphi$ with $x$ (resp. $\NOT x$)  for each $x \EQU l \in E$. Note that, $\varphi \equiv \varphi' \AND \bigwedge_{x \EQU l \in E} x \EQU l$. 

%We provide a small example to help the reader: 
\begin{example}
	Given $E = \{\neg x_1 \EQU  x_3, \NOT x_4 \EQU x_3, \NOT x_2 \EQU \NOT x_6, x_5 \EQU x_5\}$, we have $\Floor{E} = \{x_1 \EQU \NOT x_3, x_1 \EQU x_4, x_2 \EQU x_6\}$. Given $\varphi = (x_1 \OR \neg x_3 \OR x_4 \OR x_7) \AND (x_1 \OR x_3 \OR x_5) \AND (\neg x_1 \EQU  x_3) \AND (\NOT x_4 \EQU x_3) \AND (\NOT x_2 \EQU \NOT x_6) \AND (x_5 \EQU x_5)$, each literal equivalence in $\Floor{E}$ is implied. We can use $\Floor{E}$ to perform a literal-substitution to simplify $\varphi$ as $ (x_1 \OR x_7) \AND \bigwedge\Floor{E}$.
\end{example}

We propose a new notion on conjunction nodes to represent literal equivalences: 

\begin{definition}
	A {\em kernelized conjunction node} $v$ is a conjunction node consisting of a
	distinguished child, we call the {\em core} child, denoted by $ch_{\mathit{core}}(v)$,  and a set of remaining children which define equivalences, denoted by $Ch_{rem}(v)$,  such that:
	%\rytodo{updated}
	
	\begin{enumerate}[topsep=0mm,parsep=1mm]
		\item Every $w_i \in Ch_{rem}(v)$ describes a literal equivalence, i.e., $w_i = \Angle{x \EQU l}$ and the union of $\vartheta(w_i)$, denoted by $E_v$, represents a set of prime literal equivalences.
		\item For each literal equivalence $x \EQU l \in E_v$, $var(l) \notin \mathit{Vars}(ch_{\mathit{core}}(v))$.
	\end{enumerate}
\end{definition}

We now show how the model count of a kernelization of formula is related
to its core.
For simplicity, we use a sightly more general definition for model in Propositions \ref{prop:kernelizedcount}--\ref{prop:counting}. Given a formula $\varphi$ and a set of variables $X \supseteq Vars(\varphi)$, a model of $\varphi$ over $X$ is an assignment over $X$ that satisfies $\varphi$. In practice, when we want to count models for $\varphi$, we only need to make $X = \mathit{Vars}(\varphi)$.

\begin{restatable}{proposition}{kernel}\label{prop:kernelizedcount}
	%\begin{proposition}
	For a kernelized conjunction $v$ over $X$, if $\vartheta(ch_{\mathit{core}}(v))$ has $m$ models over $X$, then $\vartheta(v)$ has $\frac{m}{2^{|Ch_{rem}(v)|}}$ models over $X$.
	%\end{proposition}
\end{restatable}
\begin{proof}
	
	Given each kernelized conjunction $\varphi \wedge (x_{i_1} \leftrightarrow l_{i_1}) \wedge \cdots \wedge (x_{i_m} \leftrightarrow l_{i_m})$, we can rewrite it as a recursive form
	$\Big[\big[[\varphi \wedge( x_{i_1} \leftrightarrow l_{i_1})] \wedge (\wedge x_{i_2} \leftrightarrow l_{i_2})\big]\wedge \cdots \Big] \wedge (x_{i_m} \leftrightarrow l_{i_m})$.
	Next we show given a kernelized conjunction $\varphi = \psi \wedge (x \leftrightarrow l)$ over $X$, if $\psi$ has $m$ models over $X$, then $\varphi$ has $\frac{m}{2}$ models over $X$. By induction, we get 
	Proposition \ref{prop:kernelizedcount}.
	Without loss of generality, assume $l = x'$. 
	As this is a kernalized conjuction, $x' \notin \mathit{Vars}(\psi)$. 
	Let $\omega \cup \{x' = \mathit{false}\}$ and $\omega \cup \{x' = \mathit{true}\}$ be two assignments over $X$, where $\omega$ is a model of $\psi$ over $X \setminus \{x'\}$. 
	Since $x \EQU x'$, exactly one of the two assignments can be a model of $\varphi$, so  half of the models of $\psi$ are the models of $\varphi$.
\end{proof}
%\rytodo{Y: Shortened proof, is this OK?}

\subsection{Defining \Ctwo{}}

We begin with the widely used idea of augmenting decision diagram with conjunction in knowledge compilation~\cite{Fargier:Marquis:06,Oztok:Darwiche:14,Bart:etal:14,Lai:etal:17}. 
This idea is restated in a general form, Conjunction \& Decision Diagram, to cover our kernelization-integrated languages:

\begin{definition}\label{def:CDD}
	A Conjunction \& Decision Diagram (\CDD) is a rooted DAG wherein each node $v$ is labeled with a symbol $sym(v)$.
	If $v$ is a leaf, $sym(v) = \bot$ or $\top$.
	Otherwise, $sym(v)$ is a variable ($v$ is called a \emph{decision node})
	% (in which case $v$ is called a \emph{decision node}) 
	or operator $\AND$ (called a \emph{conjunction node}).
	Each internal node $v$ has a set of children $Ch(v)$.
	For a decision node, $Ch(v) = \{lo(u), hi(u)\}$, where $lo(u)$ ($hi(u)$) is connected by a dashed (solid) edge.
	The formula represented by a \CDD{} rooted at $u$ is defined as follows:
	
	{
		\begin{equation}
		\vartheta(u) = \begin{cases}
		\mathit{false} & sym(u) = \bot \\
		\mathit{true} & sym(u) = \top \\
		\bigwedge_{v \in Ch(u)} \vartheta(v) & sym(u) = \AND \\
		\begin{gathered}\left[\NOT sym(u) \AND \vartheta(lo(u))\right] \OR 
		\left[sym(u) \AND \vartheta(hi(u))\right] 
		\end{gathered} & \text{otherwise}
		%    \left[\NOT sym(u) \AND \vartheta(lo(u))\right] \OR 
		%    \left[sym(u) \AND \vartheta(hi(u))\right] 
		%     & \text{otherwise}
		\end{cases}
		\end{equation}
	}
\end{definition}

Hereafter we denote a leaf node by $\Angle{\bot}$ or $\Angle{\top}$, an internal node by $\Angle{sym(v), Ch(v)}$; and a decision node is denoted by $\Angle{sym(v), lo(v), hi(v)}$ sometimes. 
Given a \CDD{} rooted at $v$ (denoted by $\D_v$), its size $|\D_v|$ is defined as the number of its edges, similar to other languages in the knowledge compilation literature. 
If we admit only read-once decisions and decomposed conjunctions, then the subset of \Rzero{} is \DecDNNF.
We are now ready to describe an extension of \DecDNNF{} that captures literal equivalence, by imposing a different constraint on conjunction: 
\begin{definition}[Constrained \CDD{}, CCDD{}]\label{def:RCDD}
	A {\CDD} is called \emph{constrained} if each decision node $u$ and its decision descendant $v$ satisfy $sym(u) \neq sym(v)$, and each conjunction node $v$ is either: (i) decomposed; or (ii) kernelized.
	%  $\{\vartheta(w): w \in Ch(v)\}$ is a decomposition or kernelization of $\vartheta(v)$. 
	The language of all constrained \CDD{}s is called \Ctwo{}.
\end{definition}
%\rytodo{explain why its called constrained}

%\begin{figure}[!htb]
\begin{figure}[tb]
	\centering
	\includegraphics[width = 0.65\linewidth]{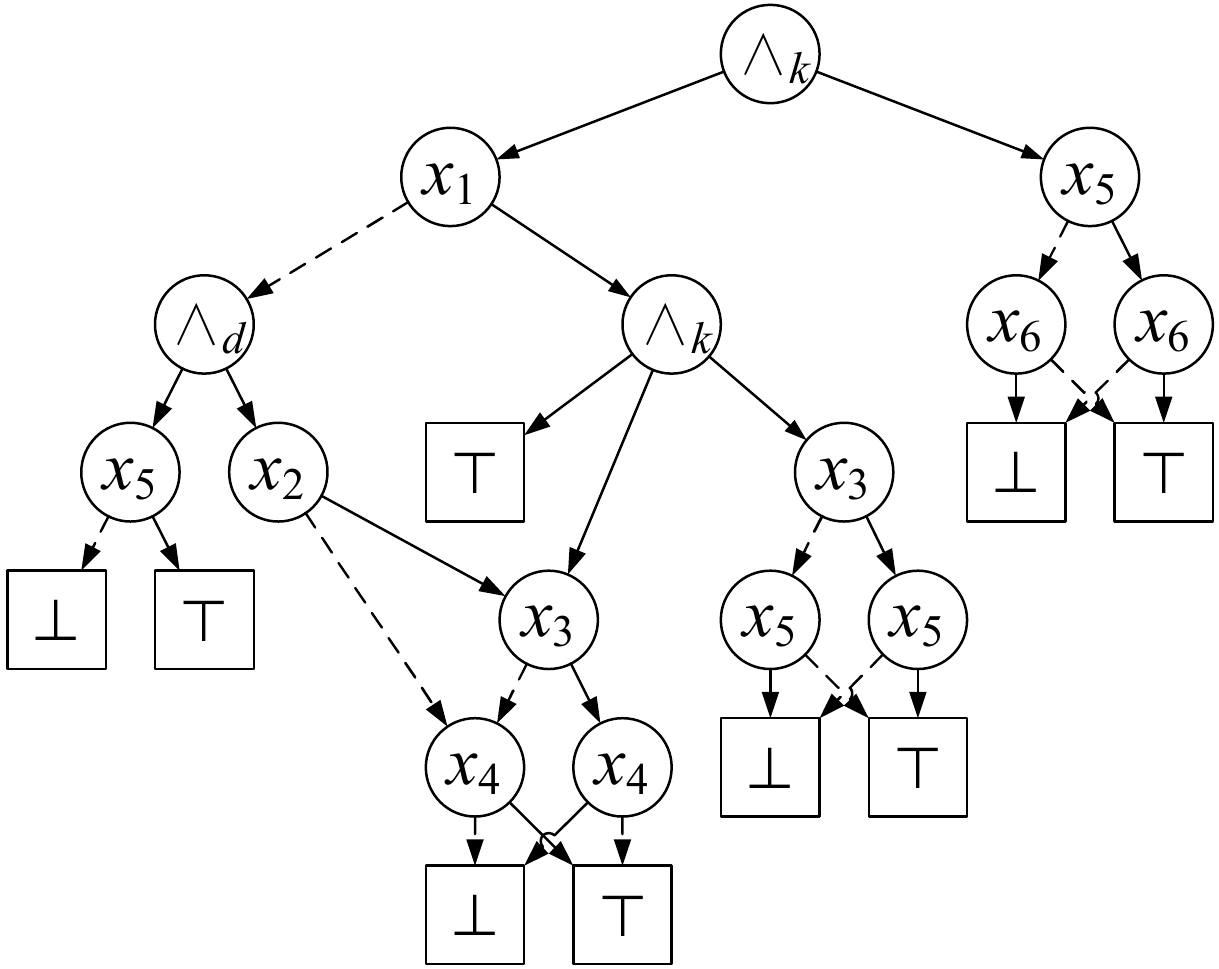}
	\caption{A diagram in \Ctwo{} representing {$(x_5 \EQU x_6) \AND \Big[\big[\NOT x_1 \AND x_5 \AND [(\NOT x_2 \AND x_4 ) \OR ( x_2 \AND (x_3 \EQU \NOT x_4) )]\big] \OR \big[ x_1 \AND ( x_3 \EQU \NOT x_4) \AND ( x_3 \EQU x_5 )\big] \Big]$, where the core child of the root is the child on the left hand side}
	}
	\label{fig:CCDD-example}
\end{figure}

% It is obvious that \Rone{} covers two existing languages \LOBDD{} and \LOBDDC{}. 
We use  $\AND_d$ and $\AND_k$ to denote decomposed and kernelized conjunctions respectively.  
%We say a node $v$ in \Rone{} if $\G_v$ is in \Rone{}.
Figure \ref{fig:CCDD-example} depicts a CCDD.
Since \DecDNNF{} is a subset of \Ctwo{} and is known to be complete, we obtain the following result on the completeness of \Ctwo{}:

\begin{restatable}{theorem}{completeness}{}\label{thm:compl}
	Given a formula, there is at least one CCDD to represent it.
\end{restatable}
\begin{comment}
To gain more tractability (in particular, tractable conditioning), we restrict the number of $\AND_k$-nodes in each path from the root to a leaf, which leads to the following desired subset of \Ctwo{} but still a superset of \DecDNNF{}. 
We remark that another representation in the knowledge compilation literature  called EDAT \cite{Koriche:etal:13} proposes a generalization of literal equivalence. Although EDAT supports both tractable model counting and conditioning, it is not a generalization of \DecDNNF{}.

\begin{definition}[\Ctwo{} with bounded $\AND_k$-depth]\label{def:RCDD}
	Given a constant integer $t$, a \CCDD{} has a kernelization depth bounded by $t$ if the maximum number of $\AND_k$-nodes appearing on each path from its root to a leaf is not greater than $t$.
	\CtwoK{t} is the set of all \CCDD{}s bounded by a $t$ kernelization depth.
\end{definition}

The diagram in Figure \ref{fig:CCDD-example} is in \CtwoK{2}.
%, while the one in Figure \ref{fig:condition-counterexample} is not bounded by a constant $\AND_k$-depth. 
It is easy to see that \CtwoK{0} is exactly \DecDNNF{}.
Since \LOBDD{} is a subset of \CtwoK{t} and is known to be complete \cite{Bryant:86}, we obtain the following result on the completeness of \CtwoK{t}:

\begin{theorem}\label{thm:compl}
	For a fixed variable ordering and constant $t$, there is at least one \CtwoK{t} to represent a given formula.
\end{theorem}
\end{comment}

\section{Tractable Model Counting and Uniform Sampling on \Ctwo{}}
\label{sec:CT-US}

In this section, we show that \Ctwo{} can support model counting and uniform sampling in polytime. We first show how \Ctwo{} supports model counting in linear time. 
We perform model counting on \Ctwo{} by a bottom-up traversal on the DAG as follows:

\begin{restatable}{proposition}{counting}{}\label{prop:counting}
	Given a node $u$ in \Ctwo{} with $\mathit{Vars}(u) \subseteq X$ and a node $v$ in $\D_u$, we use $CT(v)$ to denote the model count of $\vartheta(v)$ over $X$.
	Then $CT(u)$ can be recursively computed in linear time in $|\D_u|$:
	%\rytodo{linear time in the graph rooted at $u$}
	\begin{equation*}\label{eq:sharp-DecDNNF}
	CT(u) = \begin{cases}
	0 & sym(u) = \bot \\
	2^{|X|} & sym(u) = \top \\
	c^{-1} \cdot \prod_{v \in Ch(u)}CT(v) & sym(u) = \AND_d  \\
	\dfrac{CT(ch_{\mathit{core}}(u))}{2^{|Ch(u)| - 1}} & sym(u) = \AND_k  \\
	\dfrac{CT(lo(u)) + CT(hi(u))}{2} & {\text{otherwise}}
	\end{cases}
	\end{equation*}
	where $c = 2^{(|Ch(u)| - 1) \cdot |X|}$.
\end{restatable}
\begin{proof}
	It is easy to see the case for the leaf nodes. The case for kernelized conjunctions was discussed in Proposition \ref{prop:kernelizedcount}. For a decision node $u$, we can see that there are only half of the models over $X$ of its low (resp. high) child satisfying $\lnot sym(u) \land \vartheta(u)$ (resp. $sym(u) \land \vartheta(u)$), since $sym(u)$ does not appear in $\vartheta(lo(u))$ (resp. $\vartheta(hi(u))$). Now we discuss the case for decomposed conjunctions. Given a decomposed conjunction $u$, we show that this proposition holds when $|Ch(u)| = 2$. For the cases $|Ch(u)| > 2$, we only need to iteratively use the conclusion of the case $|Ch(u)| = 2$.
	Assume that $Ch(u) = \{v, w\}$. We can divide $X$ into three disjoint sets $X_1 = \mathit{Vars}(v)$, $X_2 = \mathit{Vars}(w)$, and $X_3 = X \setminus (X_1 \cup X_2)$. 
	Assume that $\vartheta(v)$ and $\vartheta(w)$ have $m_1$ and $m_2$ models over $X_1$ and $X_2$, respectively. Then $\vartheta(v)$ and $\vartheta(w)$ have $m_1 \cdot 2^{|X_2| + |X_3|}$ and $m_2 \cdot 2^{|X_1| + |X_3|}$ models over $X$, respectively. 
	$\vartheta(u)$ has $m_1 \cdot m_2$ models over $X_1 \cup X_2$, and has $m_1 \cdot m_2 \cdot 2^{|X_3|}$ models over $X$. It is easy to see the following equation:
	$$m_1 \cdot m_2 \cdot 2^{|X_3|} = \frac{m_1 \cdot 2^{|X_2| + |X_3|} \cdot m_2 \cdot 2^{|X_1| + |X_3|}}{2^{|X|}}$$
\end{proof}

Now we turn to uniform sampling, which is a new query in knowledge compilation \cite{KUS}.
We present the sampling algorithm on \Ctwo{} in Algorithm \ref{alg:Sample}, which takes in a consistent \CCDD{} node $u$, and returns a random model from $\satisfying{u}$.
Algorithm Sample first invokes SampleSub in Algorithm \ref{alg:SampleSub} to get a partial assignment $\omega$ of $\vartheta(u)$. If a variable $x$ does not appear in $\omega$, we will assign $x$ as a random Boolean value in lines \ref{line:Sample:supplement:begin}--\ref{line:Sample:supplement:end} via a Bernoulli distribution with parameter 0.5.
The main idea of Algorithm SampleSub is that according to the model count of each node in the \CCDD{}, we perform a random search along a subtree in the \CCDD{}, which corresponds to a partial assignment, in a top-down way.
If $u = \Angle{\top}$, SampleSub returns the empty set in line \ref{line:SampleSub:base}.
If $u$ is a decomposed node, we sample independently from its children in line \ref{line:SampleSub:decomposition}. 
If $u$ is a kernelized node, SampleSub samples from the core child first and then samples from the remaining literal equivalences (lines \ref{line:SampleSub:kernelization:begin}--\ref{line:SampleSub:kernelization:end}). 
We remark that for an equivalence node, its child $v$ represents a literal and we use $\omega(v)$ to denote the assignment on $sym(v)$ corresponding to the literal represented by $v$.
If $u$ is a decision node, SampleSub assigns a random value to $sym(u)$ according to model count ratio of low child to high child, and then samples from the chosen child (lines \ref{line:SampleSub:decision:begin}--\ref{line:SampleSub:decision:end}). 
\begin{algorithm}[!htbp]
	\caption{Sample($u$)} \label{alg:Sample}
%	\KwIn{a \Ctwo{}{} vertex $u$ such that $Vars(u) \subseteq X$}
%	\KwOut{a partial assignment $\omega$ such that $\vartheta(u)|_{\omega} \equiv true$ with probability $Pr_u(\omega)$}
	$\omega \leftarrow \textrm{SampleSub}(u, \mathit{Vars}(u))$\;
	\For { each variable $x \in \mathit{Vars}(u) \setminus \mathit{Vars}(\omega)$} {\label{line:Sample:supplement:begin}
		$b \sim \mathit{Bernoulli}(0.5)$\;
		$\omega \leftarrow \omega \cup \{x = b\}$
	}\label{line:Sample:supplement:end}
	\KwRet $\omega$\;
\end{algorithm}
\begin{algorithm}[!htbp]
	\caption{SampleSub($u$, $X$)}\label{alg:SampleSub}
		\lIf {$sym(u) = \top$} {\KwRet{$\emptyset$}}\label{line:SampleSub:base}
		\lElseIf {$sym(u) = \AND_d$} {\KwRet $\bigcup_{v \in Ch(u)} \text{SampleSub}(v, X)$}\label{line:SampleSub:decomposition}
		\uElseIf {$sym(u) = \AND_k$} {\label{line:SampleSub:kernelization:begin}
			$\omega \leftarrow \text{SampleSub}(ch_{core}(u), X)$\;
			\For { each equivalence $v \in Ch(u)$ } {
				\lIf {$(sym(v) = false) \in \omega$} {$\omega \leftarrow \omega \cup \omega(lo(v))$}
				\lElseIf {$(sym(v) = true) \in \omega$} {$\omega \leftarrow \omega \cup \omega(hi(v))$}
				\lElse {$\omega \leftarrow \omega \cup \text{SampleSub}(v, X)$}
			}
		}\label{line:SampleSub:kernelization:end}
		\Else {\label{line:SampleSub:decision:begin}
			$p = \frac{CT(hi(u), X)}{CT(lo(u), X)+CT(hi(u), X)}$\;
			$b \sim \mathit{Bernoulli}(p)$\;
			\uIf {$b = false$} {\KwRet $\{sym(u) = false\} \cup \text{SampleSub}(lo(u), X)$}
			\lElse {\KwRet{} $\{sym(u) = true\} \cup \text{SampleSub}(hi(u), X)$}
		}\label{line:SampleSub:decision:end}
\end{algorithm}

Note that according to Proposition \ref{prop:counting}, we can count models for all nodes of $\D_u$ in linear time. We assume that we finish the calling of $\oCT(u)$ before we call Sample($u$). Thus, Sample($u$) terminates in $O(|\mathit{Vars}(u)|)$ after the model count on each node is labeled.

\begin{proposition}
	Given a consistent CCDD node rooted at $u$, Sample($u$) can output each model with probability $\frac{1}{\oCT(u)}$.
\end{proposition}
\begin{proof}
If we can prove a lemma that $\text{SampleSub}(u, X)$ can output a partial assignment $\omega$ with probability $\frac{CT(\omega, X)}{CT(u, X)}$  such that $\vartheta(u)|_{\omega} \equiv true$, then it is easy to see this proposition holds. It is easy to see that this lemma holds for constant CCDDs. We assume that this lemma holds with the number of nodes $|\mathcal{N}(\D_u)| \le n$. For the case with $|\mathcal{N}(\D_u)| = n + 1$, we proceed with case analysis:
\begin{itemize}
	\item $u$ is a decomposition node: This lemma holds since the events of sampling from two different children are independent.
	\item $u$ is a kernelization node: For the case with more than one literal equivalence, we assume that $v$ is a literal equivalence in $Ch(u)\setminus \{ch_{core}(u)\}$. $u$ is equivalent to the combination of two kernelized conjunction nodes $\Angle{\land_k, \{\Angle{\land_k, Ch(u) \setminus \{v\}}, v\}}$ with less literal equivalences. Without loss of generality, we assume that $u$ has only one literal equivalence $v$. According to the induction hypothesis, SampleSub($ch_{core}(u)$, $X$) can output a partial assignment $\omega$. If $(sym(u) = false) \in \omega$, then $\omega'=\omega \cup \omega(lo(v))$ satisfies $\vartheta(u)$, and $\frac{CT(\omega, X)}{CT(ch_{core}(u), X)} = \frac{CT(\omega', X)}{CT(u, X)}$. The case $(sym(u) = true) \in \omega$ is similar to the one $(sym(u) = false) \in \omega$. For the case $sym(u) \notin \mathit{Vars}(\omega)$, it is similar to the case where $u$ is a decomposition node.
	\item $u$ is a decision node: Without loss of generality, we assume that we get $b = true$ with a probability $p = \frac{CT(hi(u), X)}{CT(lo(u), X)+CT(hi(u), X)}$. According to the induction hypothesis, $\text{SampleSub}(hi(u), X)$ can output a partial assignment $\omega$. Let $\omega' = \{sym(u) = true\} \cup \omega$. It is easy to see that $\vartheta(u)|_{\omega'} \equiv true$ and the probability of outputing $\omega'$ is $p \cdot \frac{CT(\omega, X)}{CT(hi(u), X)} = \frac{CT(\omega', X)}{CT(u, X)}$.
\end{itemize}
\end{proof}

\begin{example}
Figure \ref{fig:CCDD-example-CT} shows how to use Proposition \ref{prop:counting} to perform model counting. After all counts are marked, we can perform uniform sampling, i.e., invoking Sample($v_0$). 
Then SampleSub($v_0$, $X$) is invoked, where $X = \{x_1, \ldots, x_6\}$ will be skipped in the following explanation.
In the calling of SampleSub($v_0$), we first invoke SampleSub($v_1$). 
We perform a Bernoulli sample with probability 0.5 and assume that we obtain a $\mathit{false}$ value.
Thus, we invoke SampleSub($v_2$), and then invoke SampleSub($v_3$) and SampleSub($v_4$).
SampleSub($v_3$) returns $\{x_5 = true\}$.
In the calling of SampleSub($v_4$), we perform a Bernoulli sample with probability 0.5 and assume that we obtain a $\mathit{false}$ value.
Thus, we invoke SampleSub($v_5$), which returns $\{x_4 = true\}$.
Then we backtrack to the calling of SampleSub($v_4$) and return $\{x_2 = \mathit{false}, x_4 = true\}$, and SampleSub($v_2$) returns $\{x_2 = \mathit{false}, x_4 = true, x_5 = true\}$.
After backtracking to SampleSub($v_1$), we return $\omega_1 = \{x_1 = false, x_2 = \mathit{false}, x_4 = true, x_5 = true\}$.
Since $(x_5 = true) \in \omega_1$, SampleSub($v_0$) returns $\omega_0 = \{x_1 = \mathit{false}, x_2 = \mathit{false}, x_4 = true, x_5 = true, x_6 = true\}$.
Finally, we perform a Bernoulli sample with probability 0.5 in the calling of Sample($v_0$) and assume that $x_3$ is assigned as $true$, and therefore we obtain the sample $\{x_1 = \mathit{false}, x_2 = \mathit{false}, x_3 = true, x_4 = true, x_5 = true, x_6 = true\}$.
\begin{figure}[tb]
	\centering
	\includegraphics[width = 0.65\linewidth]{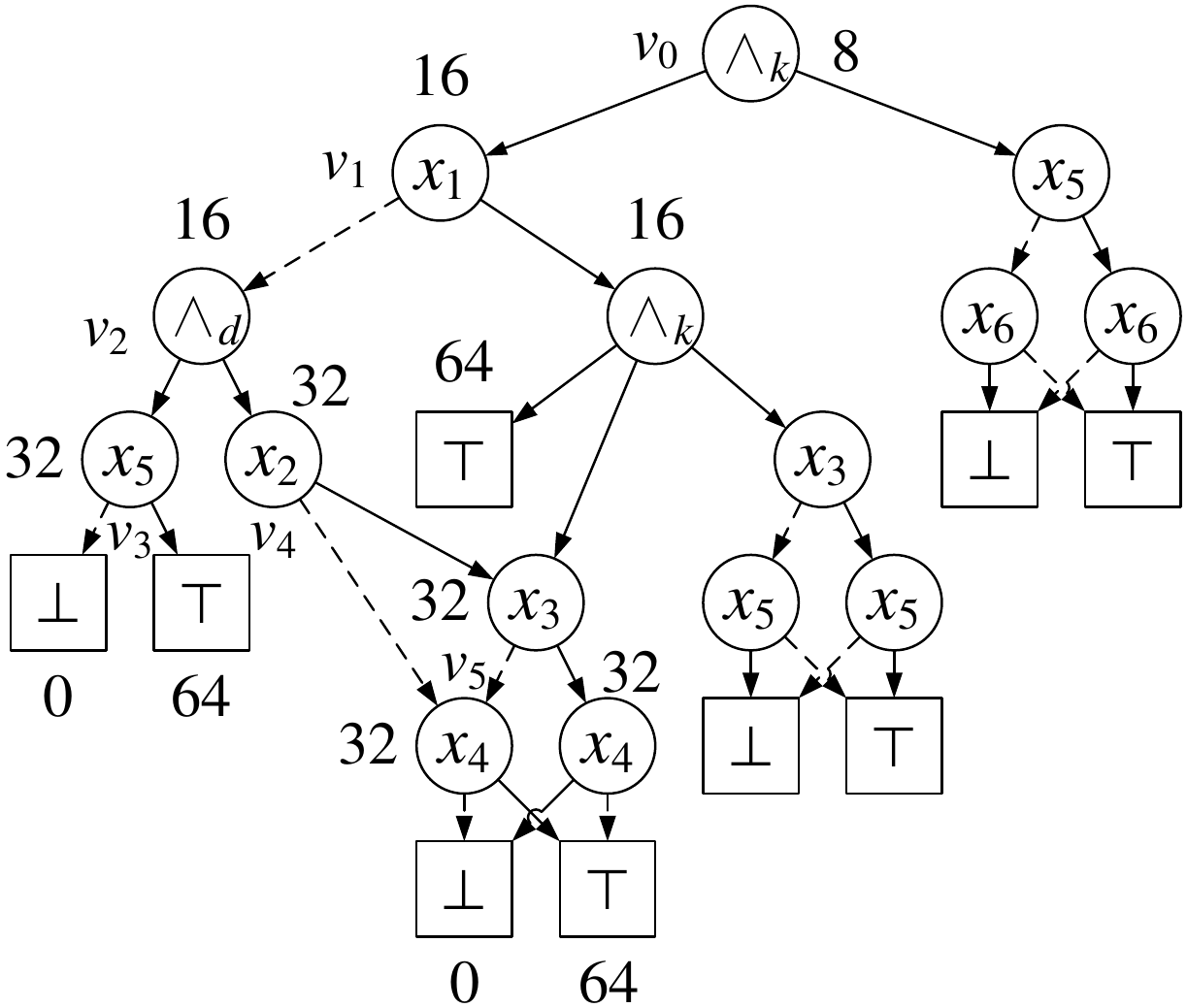}
	\caption{An illustration of performing model counting on the CCDD in Figure \ref{fig:CCDD-example}}
	\label{fig:CCDD-example-CT}
\end{figure}
\end{example}

\section{Scalable Compiler, Counter, and Sampler}
\label{sec:tools}

%Yong: You should use macros this way: {\CNF} instead of \CNF{}
% Given the tractability of {\RoneK{t}} with respect to polytime support for model counting and equivalence checking queries, it is natural to ask for the design of efficient compilers from representation languages such as \CNF{} to {\RoneK{t}}. 
In this section, we turn our attention to the compilation of a given model into \Ctwo{}, and performing model counting and uniform sampling in practice. 
We remark that in the context of knowledge compilation, there are some other languages that are generalizations of \DecDNNF{} (see e.g., Sym-\ensuremath{\mathsf{DDG}} \cite{Bart:etal:14}). As far as we know, however, there are no scalable model counters or uniform samplers reported, based on these languages. 

\subsection{\Panini{}: Compilation to \Ctwo{}}

It is standard in knowledge compilation field to design compilers that take \CNF{} as input and output an equivalent represetation corresponding to target language. In the same spirit, our algorithm, called \Panini{} and described in Algorithm~\ref{alg:Panini}, takes in a  \CNF{} formula $\varphi$, and returns a CCDD representing $\varphi$. 

We first handle the base cases lines~\ref{line:Panini:base-case-begin}--\ref{line:Panini:base-case-end}.
We return $\Angle{\bot}$ and $\Angle{\top}$ in lines \ref{line:Panini:base-case-begin} and \ref{line:Panini:base-case-end} if $\varphi$ is $\mathit{false}$ and $true$, respectively. 
We then turn to the discovery and usage of literal equivalences in the formula to perform model counting as presented in lines~\ref{line:Panini:kernel-begin}--\ref{line:Panini:kernel-end}. We use a heuristic, {\Kernelizable}, to determine whether we should spend time in detecting and using literal equivalence because those steps are themselves possibly costly. We discuss  {\Kernelizable} further in Section~\ref{sec:exactmc-implementation}. When {\Kernelizable} returns $true$, we turn to call {\DetectLitEqu} to discover literal equivalences in the formula in line~\ref{line:Panini:liteq} and 
if a non-trivial literal equivalence is discovered, we proceed to perform the compilation with respect to kernelized conjunction in lines~\ref{line:Panini:substitute}--\ref{line:Panini:kernelized-node}. 
In particular, we first invoke \ConstructCore{} to perform literal-substitution (see Section \ref{sec:CDD:lit-equ}) to obtain the formula, $\hat{\varphi}$, corresponding to the core child, and then recursively call \Panini{} over $\hat{\varphi}$. 

If no non-trivial literal equivalence is found in line~\ref{line:Panini:liteq}, then the rest of the algorithm follows the template of a \DecDNNF{} compiler. We first invoke {\Decompose} in line~\ref{line:Panini:decompose} to determine if the formula $\varphi$ can be decomposed into components. In other words, we seeks to partition the $\varphi$ into sub-formulas such that each of the components is defined over a mutually disjoint set of variables.
If such a decomposition is not 
% feasible
found,
%\rytodo{feasible$\rightarrow$found}
we pick a variable $x$ and recursively invoke \Panini{} on the residual formulas $\varphi[x \mapsto false]$ and $\varphi[x \mapsto true]$.  

\begin{algorithm}[tb]
	\caption{\Panini($\varphi$)} \label{alg:Panini}
	\lIf {$\varphi = \mathit{false}$} {\KwRet $\Angle{\bot}$} \label{line:Panini:base-case-begin}
	\lIf {$\varphi = \mathit{true}$} {\KwRet $\Angle{\top}$} \label{line:Panini:base-case-end}
	\lIf {$\Cache(\varphi) \not= nil$} {\KwRet $\Cache(\varphi)$}
	\If {$\Kernelizable(\varphi)$} {\label{line:Panini:kernel-begin}
		%		$t_{\mapsto} \leftarrow 0$, $\mathit{Units} \leftarrow \emptyset$\;
		$E \leftarrow \DetectLitEqu(\varphi)$\;\label{line:Panini:liteq}
		\If {$|\Floor{E}| > 0$}{ 
			$\hat{\varphi} \leftarrow \ConstructCore(\varphi, \Floor{E})$\; \label{line:Panini:substitute}
			%			$c \leftarrow \ExactMC(\hat{\varphi}, X, 0, \emptyset )$\; \label{line:kernel-compile}
			$v \leftarrow \Panini(\hat{\varphi})$\; \label{line:Panini:kernel}
			$V \leftarrow \{\Angle{x \EQU l} \mid x \EQU l \in \Floor{E}\}$\;
			\KwRet $Cache(\varphi) \leftarrow \Angle{\AND_k, \{v\} \cup V}$\;\label{line:Panini:kernelized-node}
		}
	}\label{line:Panini:kernel-end}
	$\Psi \leftarrow \Decompose(\varphi)$\; \label{line:Panini:decompose}
	\uIf {$|\Psi| > 1$} {
		\KwRet $\mathit{Cache}(\varphi) \leftarrow \Angle{\AND_d, \{\Panini(\psi ) \mid \psi \in \Psi\}}$\label{line:Panini:decomposition}
	}
	\Else{ 
		$x \leftarrow \FuncPickGoodVar(\varphi)$\; \label{line:Panini:decision-begin}
		$w_0 \leftarrow \Panini(\varphi[x \mapsto false])$\;
		$w_1 \leftarrow \Panini(\varphi[x \mapsto true])$\; 
		\KwRet $\mathit{Cache}(\varphi) \leftarrow \Angle{x, w_0, w_1}$\; \label{line:Panini:decision-end}
	}
\end{algorithm}

% At this point, one may wonder about the potential impact of kernelization. To this end, 
We now employ a simple example to show how kernelization helps us to reduce the size of resulting DAG. For simplicity, we assume \FuncPickGoodVar{} gives variables in the lexicographic order, and \Kernelizable{} always returns $true$ or always returns $\mathit{false}$.
\begin{example}\label{exam:Panini}
	Consider the CNF formula $\varphi$: 
	\begin{align*}
	\varphi &= (\lnot x_1 \lor \lnot x_2 \lor x_3 ) \wedge ( \lnot x_1 \lor x_2 \lor \lnot x_3) \\ &\wedge ( x_1 \lor \lnot x_2 \lor \lnot x_3) \wedge (x_1 \lor x_2 \lor x_3) \wedge (\lnot x_1 \lor \lnot x_4)\\
	&\wedge ( x_1 \lor  x_4) \wedge (\lnot x_2 \lor \lnot x_5) \wedge (x_2 \lor  x_5) 
	\end{align*}
	with $X = \{x_1, \ldots, x_5\}$. Now, there are two cases:
	\begin{description}
		\item[Without Kernelization] If {\Kernelizable} is $\mathit{false}$, \Panini{} will generate the CCDD in Figure \ref{fig:CCDD-without-k}. 
		\item[With Kernelization]  
		If {\Kernelizable} is $\mathit{true}$, we can detect two literal equivalences $x_1 \EQU \lnot x_4$ and $x_2 \EQU \lnot x_5$, and thus the residual sub-formula is equivalent to $(x_1 \oplus x_2 \oplus x_3 = 1)$. After running lines 18--20, we have two other literal equivalences $x_2 \EQU \lnot x_3$ and $x_2 \EQU x_3$. The result corresponds to the CCDD
		%		 We illustrate the trace of {\ExactMC} 
		in Figure \ref{fig:CCDD-with-k}. %When backtracking to the call corresponding to the root, we know the core child has 16 models over $\{x_1, \ldots, x_5\}$, then $\varphi$ has $\frac{16}{2^2} = 4$ models.
	\end{description}
	
\end{example}
\begin{figure}[!htbp]
	\centering
	\subfloat[CCDD without kernelization]{\label{fig:CCDD-without-k}
		\begin{minipage}[c]{0.75\linewidth}
			\centering
			\includegraphics[width = \textwidth]{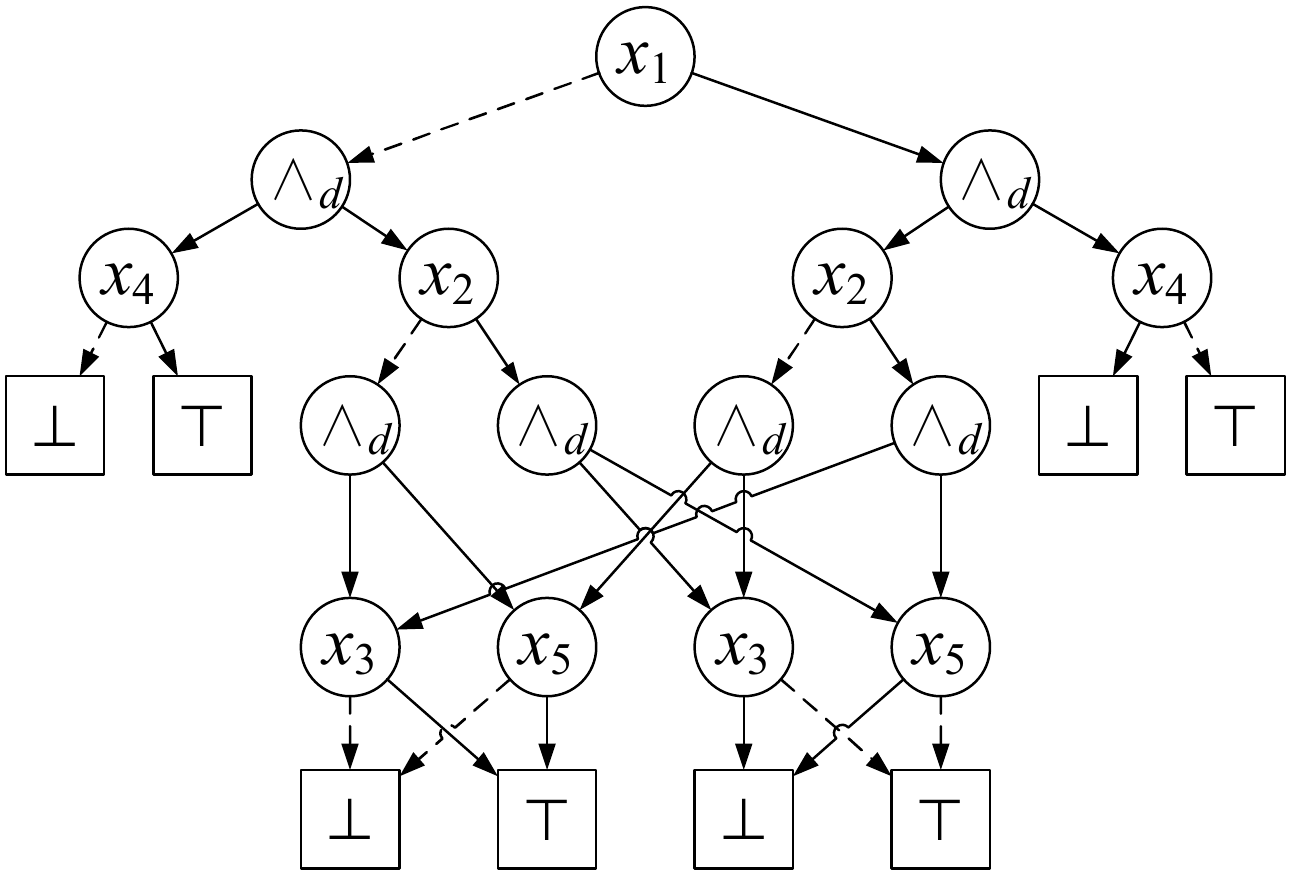}
		\end{minipage}
	}\\
	\subfloat[CCDD with kernelization]{\label{fig:CCDD-with-k}
		\begin{minipage}[c]{0.56\linewidth}
			\centering
			\includegraphics[width = \textwidth]{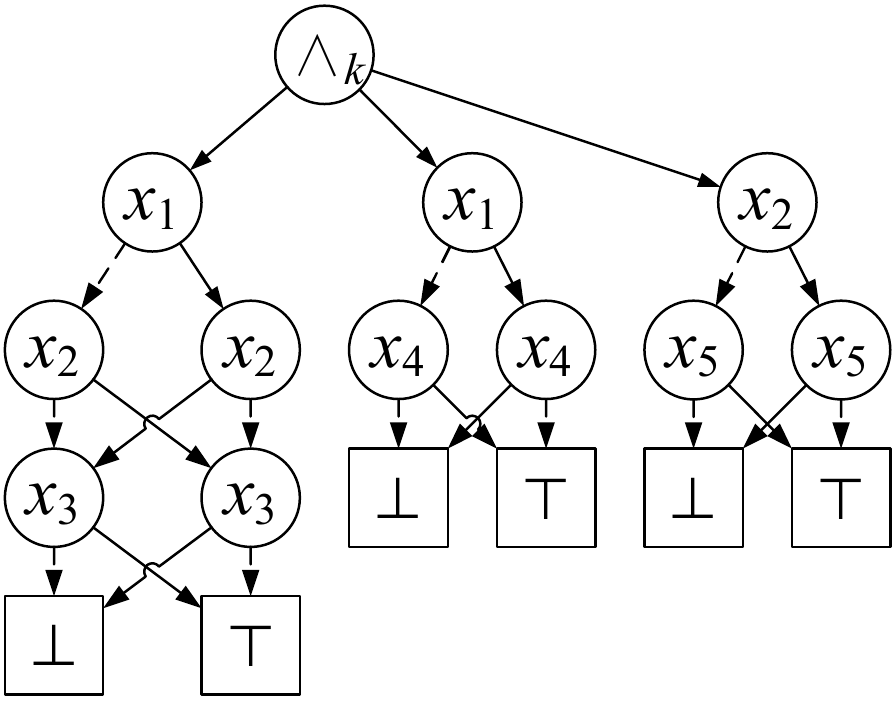}
		\end{minipage}
	}
	%	\begin{minipage}[c]{0.01\linewidth}
	%		\hspace{2mm}
	%	\end{minipage}
	
	\caption{CCDDs corresponding to compiling the formula in Example \ref{exam:Panini}}
\end{figure}

\subsection{{\ExactMC}: A Scalable Model Counter}
\label{sec:tools:ExactMC}

As discussed in Sections \ref{sec:CDD}--\ref{sec:CT-US}, {\Ctwo} has two key properties: {\Ctwo} is complete, i.e., every formula can be represented using {\Ctwo} and it supports linear model counting. 
Thus, we can immediately obtain a model counter by invoking \Panini{} and Algorithm $\mathit{CT}$.
However, we observe that it often costs much memory to store a \CCDD{} for a complex \CNF{} formula.
We do not need to actually generate the \CCDD{}, but only need to perform search with respect to \Ctwo{} to perform model counting. 
This observation motivates us to design an individual model counter, {\ExactMC}, whose trace corresponds to {\Ctwo}. 
Algorithm \ref{alg:ExactMC}, \ExactMC{}, takes in a CNF formula $\varphi$ and the set of variables $X$ (initialized to $\mathit{Vars}(\varphi)$), and returns $|\satisfying{\varphi}|$. {\ExactMC} is based on the architecture of search-based model counters, as shown in Algorithm~\ref{alg:Search}. 

We first handle the base cases lines~\ref{line:ExactMC:base-case-begin}--\ref{line:ExactMC:base-case-end} corresponding to the first two cases in Proposition \ref{prop:counting}.  
Since we are interested in computing the number of satisfying assignments over $X$, we return $2^{|X|}$ in line~\ref{line:ExactMC:base-case-end} in case $\varphi$ is $true$. 
% We remark that lines~\ref{line:base-case-begin}--\ref{line:base-case-end} correspond to the first case in Proposition \ref{alg:Counting}. 
We then turn to the discovery and usage of literal equivalences in the formula to perform model counting as presented in lines~\ref{line:ExactMC:kernel-begin}--\ref{line:ExactMC:kernel-end}. When {\Kernelizable} returns $true$, we turn to {\DetectLitEqu} to discover literal equivalences in the formula in line~\ref{line:ExactMC:liteq} and 
if a non-trivial literal equivalence is discovered, we proceed to perform exact model counting with respect to kernelized conjunction in lines~\ref{line:ExactMC:substitute}--\ref{line:ExactMC:kernel-counting} (corresponding to the fourth case in Proposition \ref{prop:counting}, where $|\Floor{E}|$ is equal to the number of children minus one). 
In particular, we first invoke \ConstructCore{} to perform literal-substitution (see Section \ref{sec:CDD:lit-equ}) to obtain the formula, $\hat{\varphi}$, corresponding to the core child, and then recursively call {\ExactMC} over $\hat{\varphi}$. 

%We remark that this two cases correspond to the first case in Proposition \ref{alg:Counting}.
If no non-trivial literal equivalence is found in line~\ref{line:ExactMC:liteq}, then the rest of the algorithm follows the template of search-based model counters. We first invoke {\Decompose} in line~\ref{line:ExactMC:decompose} to determine if the formula $\varphi$ can be decomposed into components. If such a decomposition is not 
% feasible
found,
%\rytodo{feasible$\rightarrow$found}
we pick a variable $x$ and recursively invoke {\ExactMC} on the residual formulas $\varphi[x \mapsto false]$ and $\varphi[x \mapsto true]$.  We remark that lines~\ref{line:ExactMC:product}--\ref{line:ExactMC:product2} and lines \ref{line:ExactMC:decision-begin}--\ref{line:ExactMC:decision-end} correspond to the third and fifth cases in Proposition \ref{prop:counting}, respectively. 
% \rytodo{Y: expand on \Decompose in the implementation. Explain lines 14-15 and use proposition 1 on the decomposed conjunction nodes}

\begin{algorithm}[tb]
	%\footnotesize
	%	\caption{\ExactMC($\varphi$, $X$, $t_{\mapsto}$, $\mathit{Units}$)} \label{alg:Compile}
	\caption{\ExactMC($\varphi$, $X$)} \label{alg:ExactMC}
	\lIf {$\varphi = \mathit{false}$} {\KwRet $0$} \label{line:ExactMC:base-case-begin}
	\lIf {$\varphi = \mathit{true}$} {\KwRet $2^{|X|}$} \label{line:ExactMC:base-case-end}
	\lIf {$\Cache(\varphi) \not= nil$} {\KwRet $\Cache(\varphi)$}
	\If {$\Kernelizable(\varphi)$} {\label{line:ExactMC:kernel-begin}
		%		$t_{\mapsto} \leftarrow 0$, $\mathit{Units} \leftarrow \emptyset$\;
		$E \leftarrow \DetectLitEqu(\varphi)$\;\label{line:ExactMC:liteq}
		\If {$|\Floor{E}| > 0$}{ 
			$\hat{\varphi} \leftarrow \ConstructCore(\varphi, \Floor{E})$\; \label{line:ExactMC:substitute}
			%			$c \leftarrow \ExactMC(\hat{\varphi}, X, 0, \emptyset )$\; \label{line:kernel-compile}
			$c \leftarrow \ExactMC(\hat{\varphi}, X)$\; \label{line:ExactMC:kernel-counting}
			\KwRet $\Cache(\varphi) \leftarrow \frac{c}{2 ^ {|\Floor{E}|}}$\;
		}
	}\label{line:ExactMC:kernel-end}
	$\Psi \leftarrow \Decompose(\varphi)$\; \label{line:ExactMC:decompose}
	\uIf {$|\Psi| > 1$} {
		%	$\mathit{Units} \leftarrow \mathit{Units} \cup \{l \mid l \in \Psi\}$\;
		%		$c \leftarrow \prod_{\psi \in \Psi}\{\ExactMC(\psi, X) \}$\;
		$c \leftarrow \prod_{\psi \in \Psi}\{\ExactMC(\psi, X) \}$\;\label{line:ExactMC:product}
		\KwRet $\Cache(\varphi) \leftarrow \frac{c}{2^{(|\Psi| - 1) \cdot |X|}}$\label{line:ExactMC:product2}
	}
	\Else{ 
		$x \leftarrow \FuncPickGoodVar(\varphi)$\; \label{line:ExactMC:decision-begin}
		$c_0 \leftarrow \ExactMC(\varphi[x \mapsto false], X)$\;
		$c_1 \leftarrow \ExactMC(\varphi[x \mapsto true], X)$\; 
		\KwRet $\Cache(\varphi) \leftarrow \frac{c_0 + c_1}{2}$\; \label{line:ExactMC:decision-end}
	}
\end{algorithm}

\subsection{{\ExactSamp}: A Scalable Uniform Sampler}
\label{sec:tools:US}

Since {\Ctwo} is complete and supports tractable model counting and uniform sampling, we can immediately obtain a uniform sampler by invoking Algorithms \Panini{}, $\mathit{CT}$, and Sample. The workflow of our uniform sampler, called {\ExactSamp}, is as follows:

\begin{itemize}
	\item First, we invoke \Panini{} to transform a \CNF{} formula into an equivalent CCDD;
	\item Second, we invoke $\mathit{CT}$ to label model count for each node in the CCDD; and
	\item Finally, we invoke Algorithm Sample $s$ times on the CCDD to generate $s$ identically and independently distributed samples.
\end{itemize}
Among the above three steps, the first one is often the most time-consuming but can be performed \emph{offline}. The compiling time is amortized over \emph{online} callings of Algorithm Sample for sample generation.
In practice, this setting facilitates the end-user (e.g.,
verification engineer who typically invokes a sampler repeatedly till a bug is triggered \cite{Naveh:etal:06}).

\subsection{Implementation}\label{sec:exactmc-implementation}

Since the core contribution of our work lies in the on-the-fly construction and usage of kernelized conjunction nodes, we now discuss the implementation details that are crucial for runtime efficiency of our tools. As is the case for most heuristics in SAT solving and related communities, we selected parameters empirically. Given the original formula $\varphi$, we will use $\#\mathit{NonUnitVars}$ to denote the number of variables appearing in the non-unit clauses of $\varphi$.
\begin{description}
	\item[{\Kernelizable}] As mentioned earlier, the detection and usage of literal equivalences can be significantly advantageous but our preliminary experiments indicated the need for caution. In particular, we observed that the implicit construction of kernelized conjunction node over the trace was not helpful for {\em easy} instances. To this end, we rely on the number of variables as a proxy for the hardness of a formula, in particular at every level of recursion, we classify a formula $\varphi$ to be easy if  $|Vars(\varphi)| \le \mathit{easy\_bound}$, where $\mathit{easy\_bound}$ is defined by $\min(128, \#\mathit{NonUnitVars} / 2)$. If the formula $\varphi$ is classified as easy, then {\Kernelizable} returns $\mathit{false}$. Else, we consider the search path from the last kernelization (if no kernelization, then the root) to the current node. If the number of unit clauses on the path is greater than 48 and also greater than twice the number of decisions on the path, {\Kernelizable} returns $true$. The intuition behind the usage of unit clauses is that unit clauses are often useful to simplify the current sub-formula and thus possibly lead to many literal equivalences. In the other cases,  \Kernelizable{} returns $\mathit{false}$. We empirically determine the heuristic to have good performance.
	
	\item[{\DetectLitEqu}]  Recall, we need to check for a chosen pair of literals $l_1$ and $l_2$, 
	whether $l_1 \EQU l_2$ is a literal equivalence implied by $\varphi$ in \DetectLitEqu{}. For an efficient check, we rely on using implicit Boolean Constraint Propagation (i-BCP) for the assignments $l_1 \wedge \neg l_2$ and $\neg l_1 \wedge l_2$. The usage of i-BCP in model counting dates back to sharpSAT~\cite{Thurley2006}. We perform some simplications on each component in order to detect more literal equivalences
	which includes removing literals from clauses, and unnecessary clauses. 
	In particular, we designed a pre-processor called \PreLite\footnote{
		We remark that the design of \PreLite{} is similar to the pmc \cite{pmc}  pre-processor.
	}
	to perform the initial kernalization on the original formula. 
	%	$l_1 \wedge \neg l_2$ and $\neg l_1 \wedge l_2$. In such a case, we can conclude that $\varphi \models (l_1 \leftrightarrow l_2)$. 
	\item[Prime Literal Equivalences] We employ union-find sets to represent prime literal equivalences, which allows us to efficiently compute prime literal equivalences from a set of literal equivalences. 
	\item[Decision Heuristics] We combine the widely used heuristic minfill \cite{book:Darwiche:09}
	and a new dynamic ordering, which we call 
	{\em dynamic combined largest product} (DLCP) to pick good variables. Given a variable, the DLCP value is the product of the weighted sum of negative appearances and 
	% the number of 
	positive appearances of the variable. 
	Given an appearance, the heuristic considers the following cases:
	(i) if it is in an original binary clause, the weight is 2; 
	(ii) if it is in a learnt binary clause, the weight is 1; 
	(iii) if it is in an original non-binary clause with $m$ literals, the weight is $\frac{1}{m}$; 
	otherwise, (iv) the weight is 0.
	If the minfill treewidth is greater than a crossover constant $\min(128, \#\mathit{NonUnitVars}/c)$,
	% 128 or one-seventh of the number of variables appearing in non-unit clauses in $\varphi$, 
	we use DLCP, otherwise, minfill.
	We choose $c = 5$ for compilation and $c = 7$ for counting.
	We observed in the experiments that for an instance with high treewidth, DLCP is often useful to lead to a sub-formula with many literal equivalences after assigning some variables.
	%	\item Finally, as is the norm for \#SAT tools, the practical implementation of {\ExactMC} also benefits from standard search techniques such as conflict clause learning, and dynamic decomposition (i.e., the \Decompose{} function), and component caching. 
	%	
\end{description}

\section{Experimental Evaluation}
\label{sec:experiments}

We implemented prototypes of \Panini{}, \ExactMC{}, \ExactSamp{} in C++.
We evaluated these tools \footnote{\Panini{}, \ExactMC{} and \PreLite{} will be available at \url{https://github.com/meelgroup/KCBox}}
%Our prototype \panini{} implementation is in C++. 
%We evaluated
% performed empirical evaluation 
on a comprehensive set of 1114 benchmarks
%~\cite{cril} 
\footnote{The benchmarks are from the following sites: \\
\url{https://www.cril.univ-artois.fr/KC/benchmarks.html} \\
\url{https://github.com/meelgroup/sampling-benchmarks} \\
\url{https://github.com/dfremont/counting-benchmarks} \\ 
\url{https://www.cs.ubc.ca/hoos/SATLIB/benchm.html}
}
from a wide range of application areas, including automated planning, Bayesian networks, configuration, combinatorial circuits, inductive inference, model checking, program synthesis, and quantitative information flow (QIF) analysis. These instances have been employed in the past to evaluate model counting and knowledge compilation techniques~\cite{D4,Lai:etal:13,Lai:etal:17,Ganak,Fremont:etal:17}.
% In this section, we first evaluate the performance of compiling time and compilation sizes of \panini{}; and then apply \panini{} to then 
% a reasoning task applicable to probabistic methods, 
% e.g. probabilistic databases, probabilistic programming and tractable learning.
% \todo{some benchmarks are not available}
The experiments were run on a cluster\footnote{
The cluster is a typical HPC cluster where jobs are run through a job queue.
}
where each node has 2xE5-2690v3 CPUs with 24 cores and 96GB of RAM. 
Each instance was run on a single core with a timeout of 3600 seconds and 4GB memory.

\subsection{Knowledge Compilation}

\begin{table}[tb]
	\centering
	\footnotesize
	\aboverulesep = 0.2ex
	\belowrulesep = 0.2ex
	%\captionsetup{width = 95mm}
	%\belowcaptionskip = 125mm
	%\renewcommand{\arraystretch}{1.5}
	\begin{tabularx}{\linewidth}{l*{6}{>{\centering\arraybackslash}X}}\toprule
		\multirow{2}*{domain (\#)} & \LOBDDC{}{} & \LSDD{} & \multicolumn{3}{c}{\DecDNNF{}} & \Ctwo{} \\\cmidrule{4-6}
		{} & BDDC & miniC2D & c2d & Dsharp & \dfour{} & \Panini{} \\\midrule
		%domain (\#) & \LOBDDC{}{}/LLY & \LSDD{}/miniC2D & \RoneK{1}/Panini & \dDNNF{}/dsharp &  \dDNNF{}/d4 \\\midrule
		Bayesian-Networks (201) & 157 & 153 & \textbf{164} & 113 & 154 & 161 \\
		BlastedSMT (200) & 158 & 164 & \textbf{165} & 135 & 163 & 161 \\
		Circuit (56) & 35 & 33 & 34 & 30 & 37 & \textbf{41} \\
		Configuration (35) & 31 & 29 & \textbf{35} & 21 & 32 & 32 \\
		Inductive-Inference (41) & 15 & 15 & \textbf{19} & 15 & 15 & \textbf{19} \\
		Model-Checking (78) & 66 & 68 & 72 & 46 & 74 & \textbf{76} \\
		Planning (243) & 188 & 168 & \textbf{192} & 144 & 187 & \textbf{192} \\
		Program-Synthesis (221) & 85 & 57 & 63 & 61 & 77 & \textbf{89} \\
		QIF (39) & 10 & 6 & \textbf{17} & 5 & 8 & 13 \\\midrule[0.02em]
		Total (1114) & 745 & 693 & 761 & 570 & 747 & \textbf{782} \\\bottomrule
	\end{tabularx}
	\caption{Compiling performance between \LOBDDC{}{}, \LSDD{}, \dDNNF{}, and \Ctwo{}, where each cell below language \ensuremath{\mathsf{L}} refers to the number of instances compiled successfully into target \ensuremath{\mathsf{L}}
	} \label{tab:expri:compilation}
	% \vspace*{-5mm}
\end{table}

\begin{table*}[tb]
	\scriptsize
	\centering
	%\footnotesize
	\aboverulesep = 0.2ex
	\belowrulesep = 0.2ex
	%\captionsetup{width = 95mm}
	%\belowcaptionskip = 125mm
	%\renewcommand{\arraystretch}{1.5}
	\begin{tabularx}{\linewidth}{l*{7}{>{\centering\arraybackslash}X}}\toprule
		\multirow{2}*{domain/instance} & \multirow{2}*{BDDC} & \multirow{2}*{miniC2D} & \multirow{2}*{c2d} & \multirow{2}*{Dsharp} & \multirow{2}*{D4} & \multicolumn{2}{c}{Panini} \\ \cmidrule{7-8}
		{} & {} & {} & {} & {} & {} & size & \#knodes \\\midrule
		%domain (\#) & \LOBDDC{}{}/LLY & \LSDD{}/miniC2D & \RoneK{1}/Panini & \dDNNF{}/dsharp &  \dDNNF{}/d4 \\\midrule
		Bayesian-Networks/50-20-9-q  & 1.6e6 & 5.2e6 & 2.0e6 & -- & -- & 6.2e5 & 6.2e4 \\
		BlastedSMT/squaring12 & -- & 8.4e7 & -- & -- & 4.7e8 & 5.4e6 & 1.2e5 \\
		Circuit/s13207.1 & -- & -- & -- & -- &  & 1.9e5 & 1.8e4 \\
		Configuration/C210\_FS & -- & -- & 2.2e7 & -- & -- & -- & -- \\
		Inductive-Inference/ii32b1 & -- & -- & 1.7e7 & -- & -- & 1.0e7 & 0 \\
		Model-Checking/bmc-galileo-8 & -- & -- & -- & 1.4e6 & 1.3e7 & 8.0e7 & 6.7e2 \\
		Planning/blocks\_right\_4\_p\_t6 & -- & -- & -- & -- & -- & 4.2e7 & 1.0e6 \\
		Program-Synthesis/sygus\_09A-1 & -- & -- & -- & -- & -- & 4.2e7 & 1.6e5 \\
		%QIF & min-32s & 839.14 & 731.71 & -- & -- & -- & 219.0 & 1 & 1  \\
		QIF/min-16s & -- & -- & 1.3e8 & -- & -- & -- & --  \\
		\bottomrule
	\end{tabularx}
	\caption{Compilation statistics on selected instances, where \textquoted{--} denotes timeout or out of memory, \textquoted{\#knodes} denotes the total number of kernelized nodes, and the other columns are about compilation size}
	%	under the given conditions} 
	\label{tab:expri:compilation:kernel}
\end{table*}

We compared \Panini{} with state-of-the-art compilers for 
the following target languages:
(i) \LSDD{} with miniC2D \cite{Oztok:Darwiche:15}; 
(ii) \LOBDDC{}{} with BDDC \cite{Lai:etal:17}; 
(iii) \DecDNNF{} with c2d \cite{c2d}, Dsharp \cite{Dsharp} and \dfour{} \cite{D4}. 
We used the widely employed pre-processing tool pmc~\cite{pmc} for all the instances, which preserves the equivalence between input instance and pre-processed instance and is quite helpful for improving the efficiency of knowledge compilers.
We employed the minfill heuristic for variable ordering in BDDC, miniC2D, and c2d, which has been shown to significantly improve runtime and space performance~\cite{Dsharp,Lai:etal:17}. 
Dsharp and {\dfour} employ their own custom variable ordering heuristics, which were shown to improve their performance \cite{Dsharp,D4}. 

Table \ref{tab:expri:compilation} shows the total performance of the six compilers compiling from \CNF{} to the target language. Overall, \Panini{} compiled 37, 89, 21, 212, and 35 more instances than BDDC, miniC2D, c2d, Dsharp, and \dfour{}, respectively. 
We remark that \Panini{} compiled 58, 168, 43, 251, and 44 more instances than BDDC, miniC2D, c2d, Dsharp, and D4 respectively without the usage of pmc.
Figures \ref{fig:compilation:time} and \ref{fig:compilation:sizes} show the cactus plots for runtime and compilation sizes (in terms of edges in the DAG) for all six compilers.
The $x$-axis gives the number of benchmarks; and
the $y$-axis is compiling time (resp. compilation sizes), i.e., 
a point $(x, y)$ in Figure~\ref{fig:compilation:time} shows that $x$ benchmarks took less than or equal to $y$ seconds to compile. 
The results show that \Panini{} can give start-of-the-art compilation
both in runtime and compiled size.
We show the space performance of \Ctwo{} on some selected instances in Table \ref{tab:expri:compilation:kernel}.
% The experimental results show that \Ctwo{} has obvious space advantage with the use of kernelized nodes.
The experimental results show that \Ctwo{} has obvious space advantage and there are many kernelized nodes in the compiled forms.

\begin{figure*}[tb]
	\centering
	\subfloat[Compilation Time]{\label{fig:compilation:time}
		\begin{minipage}[c]{0.5\linewidth}
			\centering
			\includegraphics[width = \textwidth]{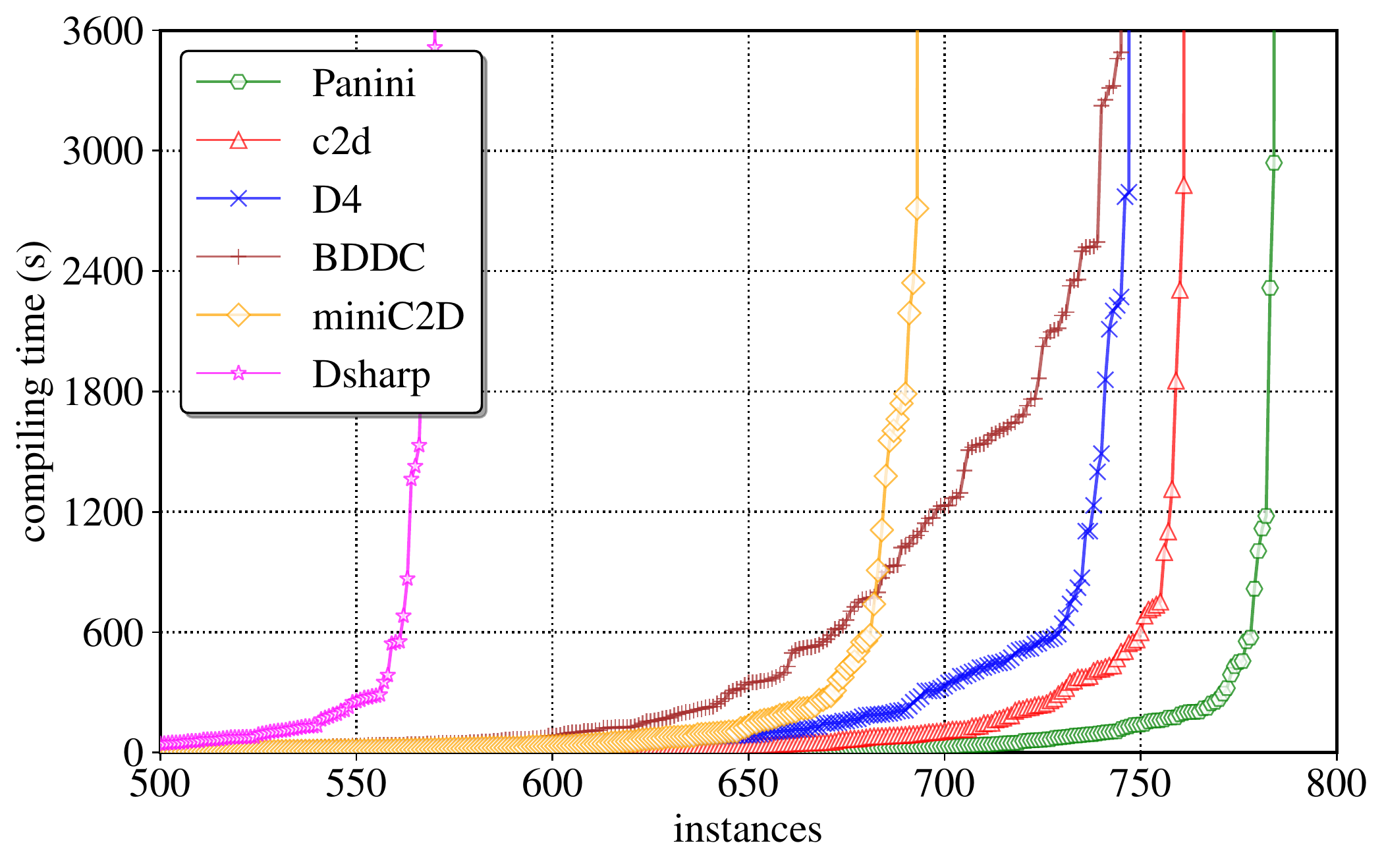}
		\end{minipage}
	}
	\subfloat[Compilation Size]{\label{fig:compilation:sizes}
		\begin{minipage}[c]{0.5\linewidth}
			\centering
			\includegraphics[width = \textwidth]{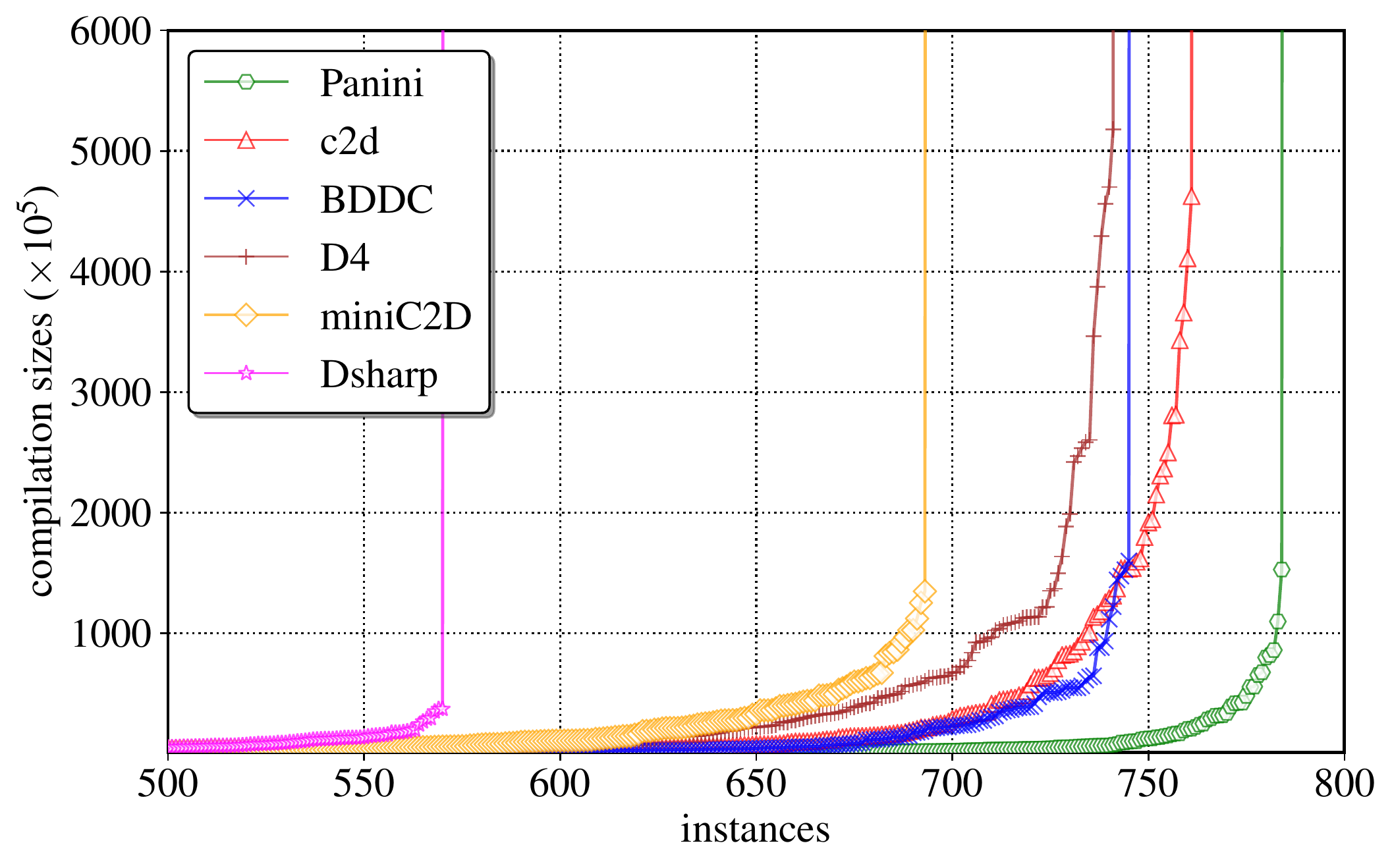}
		\end{minipage}
	}
	\caption{Cactus plots comparing the performance of different compilers. (Best viewed in color)}\label{fig:compilation}
\end{figure*}

\subsection{Model Counting}

We compared \ExactMC{} with state-of-the-art exact counters from each of the three paradigms: compilation-based, search-based or variable elimination-based.
Compilation-based counters include c2d and D4 based on \DecDNNF{}.  
For search-based counters, we compared with Ganak \cite{Ganak} and SharpSAT-TD \cite{SharpSAT-TD}, the winners of the unweighted tracks in model counting competitions 2020 and 2021 \footnote{See \url{https://mccompetition.org/past_iterations} for detailed information about model counting competition.}, respectively.
We remark that Ganak is a recent probabilistic exact model counter that implicitly combines \DecDNNF{} approach with probabilistic hashing to provide exact model count with a given confidence $1-\delta$ (we used the default $\delta=0.05$).
Note that probabilistic exact is a stronger notion than another related notion of probabilistic approximate counting~\cite{Chakraborty:etal:19}. Also, perhaps it is worth remarking that Ganak and SharpSAT-TD builds on and was shown to significantly improve upon the prior state of the art search-based counter, sharpSAT \cite{sharpSAT}. 
For variable elimination-based counters, we compared with ADDMC
% We also compared \panini{} with a variable elimination-based exact counter 
\cite{ADDMC}.

We used the widely employed pre-processing tool B+E~\cite{BplusE} for all the instances, which was shown more powerful in model counting than pmc \cite{BplusE,Ganak}, but does not preserve the equivalence between input instance and pre-processed instance.
We remark that B+E can often simplify almost all of the literal equivalences in the original formula detected by i-BCP. We emphasize that the literal equivalences in \ExactMC{} is a  ``in-processing technology'', and since B+E is already used, the literal equivalences used in \ExactMC{} are basically the ones  appearing in the sub-formulas. 
%\rytodo{revised - Y is that correct?}
Consistent with recent studies, we excluded the preprocessing time from the solving time for each tool as preprocessed instances were used on all solvers.
We emphasize that the usage of pre-processing favors other competing tools than \ExactMC{}, except SharpSAT-TD where a pre-processor similar to B+E has been integrated. To see the effect of B+E, Ganak, c2d, SharpSAT-TD, \dfour{}, ADDMC, and \ExactMC{} solved 173, 117, 2, 170, 283, and 54 less instances 
without the pre-processing, respectively.
Similarly, we employed the minfill heuristic for variable ordering in c2d. {\dfour}, Ganak, and SharpSAT-TD employ their own custom variable ordering heuristics, which were shown to improve their performance \cite{D4,Ganak}. 

% As discussed in Section~\ref{sec:kc-map}, 
% therefore, we would like $t$ to be as small as 1. Therefore, we set $t$ to 1 for our empirical evaluation.
%\todo{Is this for variable/vtree ordering?}.
%We used the widely employed pre-processing tool pmc~\cite{pmc} for all the instances. The usage of pre-processing favors other competing tools than {\panini}: In particular, miniC2D, \dsharp{}, LLY, \dfour{} and \panini{} solved 178, 24, 19, 8 and 7 more instances respectively after pre-processing. 

\begin{table*}[tb]
	\centering
	\footnotesize
	\aboverulesep = 0.2ex
	\belowrulesep = 0.2ex
	%\captionsetup{width = 95mm}
	%\belowcaptionskip = 125mm
	%\renewcommand{\arraystretch}{1.5}
	\begin{tabularx}{\linewidth}{l>{\centering\arraybackslash}c*{5}{>{\centering\arraybackslash}X}}\toprule
		\multirow{2}*{domain (\#)} & \multirow{2}*{ADDMC} & \multicolumn{4}{c}{\DecDNNF{}} & \Ctwo{} \\\cmidrule{3-6}
		{} & {} & Ganak & c2d & SharpSAT-TD & \dfour{} & ExactMC \\\midrule
		%domain (\#) & \LOBDDC{}{}/LLY & \LSDD{}/miniC2D & \RoneK{1}/Panini & \dDNNF{}/dsharp &  \dDNNF{}/d4 \\\midrule
		Bayesian-Networks (201) & \textbf{191}  & 170 & 183 & 186 & 179 & 186 \\
		BlastedSMT (200) & 166 & 163 & 160 & 163 & 162 & \textbf{169} \\
		Circuit (56) & 45 & 49 & 50 & 50 & 49 & \textbf{51} \\
		Configuration (35) & 21 & \textbf{35} & \textbf{35} & 32 & 33 & 31 \\
		Inductive-Inference (41) & 3 & 18 & 19 & 18 & 18 & \textbf{22} \\
		Model-Checking (78) & 64 & 73 & \textbf{74} & 73 & 72 & \textbf{74} \\
		Planning (243) & 187 & 207 & 209 & 212 & 206 & \textbf{213} \\
		Program-Synthesis (220) & 52 & 96 & 76 & 77 & 90 & \textbf{108} \\
		QIF (40) & 24 & \textbf{32} & \textbf{32} & 28 & 26 & \textbf{32} \\\midrule[0.02em]
		Total (1114) & 753 & 843 & 838 & 839 & 835 & \textbf{886} \\\bottomrule
	\end{tabularx}
	\caption{Comparative counting performance between Ganak, c2d, SharpSAT-TD, \dfour{}, and \ExactMC{}, where each cell below tool refers to the number of solved instances}
%	under the given conditions} 
\label{tab:expri:counting}
\end{table*}

\begin{table*}[tb]
	\scriptsize
	\centering
	%\footnotesize
	\aboverulesep = 0.2ex
	\belowrulesep = 0.2ex
	%\captionsetup{width = 95mm}
	%\belowcaptionskip = 125mm
	%\renewcommand{\arraystretch}{1.5}
	\begin{tabularx}{\linewidth}{l*{6}{>{\centering\arraybackslash}X}}\toprule
		\multirow{2}*{domain/instance} & \multirow{2}*{Ganak} & \multirow{2}*{c2d} & \multirow{2}*{SharpSAT-} & \multirow{2}*{D4} & \multicolumn{2}{c}{ExactMC} \\ \cmidrule{6-7}
		{} & {} & {} & {TD} & {} & time & \#kers \\\midrule
		%domain (\#) & \LOBDDC{}{}/LLY & \LSDD{}/miniC2D & \RoneK{1}/Panini & \dDNNF{}/dsharp &  \dDNNF{}/d4 \\\midrule
		Bayesian-Networks/Grids\_11  & 1239.5 & -- & 395.2 & -- & 915.9 & 0 \\
		BlastedSMT/blasted\_case138 & -- & -- & -- & -- & 0.9 & 24 \\
		Circuit/2bitadd\_11 & -- & -- & -- & -- & 2724.1 & 11580 \\
		Configuration/C168\_FW & 338.6 & 14.0 & 133.4 & 68.3 & -- & -- \\
		Inductive-Inference/ii32d2 & -- & -- & 708.4 & -- & 604.2 & 559 \\
		Model-Checking/bmc-galileo-8 & 1.3 & 2145.9 & -- & -- & 1.8 & 33 \\
		Planning/logistics.c & 214.4 & 536.7 & 182.3 & 173.5 & 29.1 & 7366 \\
		Program-Synthesis/sygus\_09A-1 & -- & -- & -- & -- & 161.0 & 20403 \\
		QIF/min-2s & 61.3 & 0.3 & 131.7 & 125.4 & 10.1 & 8  \\
		\bottomrule
	\end{tabularx}
	\caption{Counting statistics on selected instances using Decision-DNNF-based and \Ctwo{}-based counters, where \textquoted{--} denotes timeout or out of memory, \textquoted{\#kers} denotes the total number of kernelizations, and the other columns are about solving time in seconds}
%	under the given conditions} 
\label{tab:expri:kernel}
\end{table*}

\begin{figure}[tb]
	\centering
	\includegraphics[width = 0.5\linewidth]{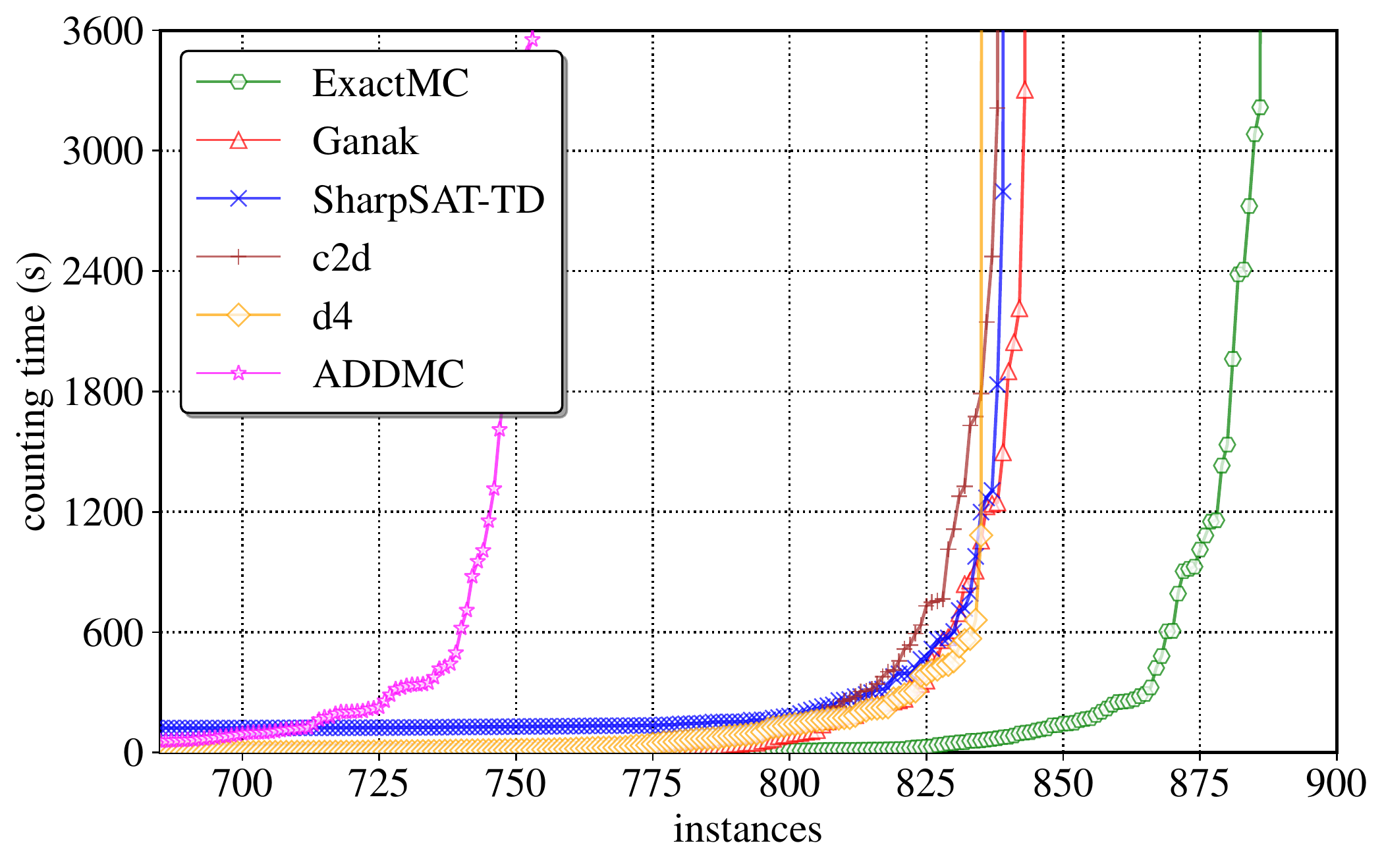}
	\caption{Cactus plot comparing the solving time of different counters. (Best viewed in color)}\label{fig:counting:time}
\end{figure}

Table \ref{tab:expri:counting} shows the performance of the six counters. Overall, \ExactMC{} solved 133, 43, 48, 47, and 51 more instances than ADDMC, Ganak, c2d, SharpSAT-TD, and \dfour{}, respectively.  
Upon closer inspection of the performance of various tools across different domains, we observe that \ExactMC{} performed the best on seven out of nine domains.
% It is perhaps emphasizing that while {\panini} provides deterministic exact model count, the answers computed by Ganak are probabilistic.  
Figure \ref{fig:counting:time} shows the cactus plot for runtime for all the six tools.
The $x$-axis gives the number of benchmarks; and
the $y$-axis is running time, i.e., 
a point $(x, y)$ in Figure~\ref{fig:counting:time} shows that $x$ benchmarks took less than or equal to $y$ seconds to solving. 
%For compilation sizes, a point $(x, y)$ implies that the compilation sizes of  $x$ benchmarks are less than or equal to $y$. 
The results show that \ExactMC{} can improve the state-of-the-art model counting
across all three paradigms.

We remark that all of Ganak, c2d, SharpSAT-TD, and D4 perform searches with respect to \DecDNNF{}.
In order to show the effect of kernelization, we compared \ExactMC{} with the virtual best solver of c2d, D4, Ganak, and SharpSAT-TD (VBS-DecDNNF). We found that even in such an extreme case, \ExactMC{} solved one more instance than VBS-DecDNNF. 

We present the effect of kernelization on some selected instances and solving
times in Table \ref{tab:expri:kernel}. 
The experimental results show that for some instances (e.g., blasted\_case138), even a small number of kernelizations are very useful to accelerate solving. Furthermore, it is worth noticing that we are able to perform a large number of kernelizations in the benchmarks, 
% thereby validating our initial hypothesis regarding existence of literal equivalence in sub-formulas despite the pre-processing. 
showing that substantial literal equivalence can occur in sub-formulas despite the use of pre-processing, e.g. sygus\_09A-1 (Program-Synthesis). 
We also conducted experiments where kernelization was disabled in \ExactMC{} (without lines 4--12 in Algorithm \ref{alg:ExactMC}). We found that the resulting counter solved 17 less instances than the original version of \ExactMC{}, and the average PAR-2 score increased to 1603 from 1505.\footnote{The average PAR-2 scoring scheme gives a penalized average runtime, assigning a runtime of two times the time limit (instead of a ``unsolved'' status) for each benchmark not solved by a tool.}
\subsection{Uniform Sampling}

To the best of our knowledge, SPUR and KUS are the only two tools that can perform sampling on {\CNF} formulas with theoretical guarantees of uniformity.
SPUR was built on top of sharpSAT, while KUS employs {\dfour} to perform \DecDNNF{} compilation.
Consistent with the previous studies, we compare {\ExactSamp} with SPUR and KUS on the generation of 1000 samples for each instance.
As with the compilation experiments, we use pmc to pre-process the instances
as it preserves equivalence.
Table \ref{tab:expri:sampling} shows the performance of SPUR, KUS, and {\ExactSamp}. 
Overall, {\ExactSamp} solved 132 and 186 more instances than SPUR and KUS, respectively, and performed the best on all the (nine) domains.
We remark that {\ExactSamp} solved 157 and 201 more instances than SPUR and KUS, respectively, without the usage of pmc.
Figure \ref{fig:sampling:time} shows the cactus plot for runtime for all three samplers.
The results also demonstrate the significant improvement of {\ExactSamp} compared with SPUR and KUS.

\begin{table*}[tb]
	\centering
	\footnotesize
	\aboverulesep = 0.2ex
	\belowrulesep = 0.2ex
	%\captionsetup{width = 95mm}
	%\belowcaptionskip = 125mm
	%\renewcommand{\arraystretch}{1.5}
	\begin{tabularx}{\linewidth}{l*{3}{>{\centering\arraybackslash}X}}\toprule
		\multirow{2}*{domain (\#)} & \multirow{2}*{SPUR} & \DecDNNF{} & \Ctwo{} \\
		{} & {} & KUS & {\ExactSamp} \\\midrule
		Bayesian-Networks (201) & 132 & 109 & \textbf{161} \\
		BlastedSMT (200) & 147 & 137 & \textbf{161} \\
		Circuit (56) & 32 & 30 & \textbf{41} \\
		Configuration (35) & 28 & 23 & \textbf{32} \\
		Inductive-Inference (41) & 16 & 15 & \textbf{18} \\
		Model-Checking (78) & 54 & 63 & \textbf{76} \\
		Planning (243) & 159 & 152 & \textbf{192} \\
		Program-Synthesis (221) & 73 & 59 & \textbf{89} \\
		QIF (39) & 7 & 6 & \textbf{12} \\\midrule[0.02em]
		Total (1114) & 648 & 594 & \textbf{780} \\\bottomrule
	\end{tabularx}
	\caption{Comparative sampling performance between SPUR, KUS, and \ExactSamp{}, where each cell below tool refers to the number of solved instances}
	%	under the given conditions} 
	\label{tab:expri:sampling}
\end{table*}

\begin{figure}[tb]
	\centering
	\includegraphics[width = 0.5\linewidth]{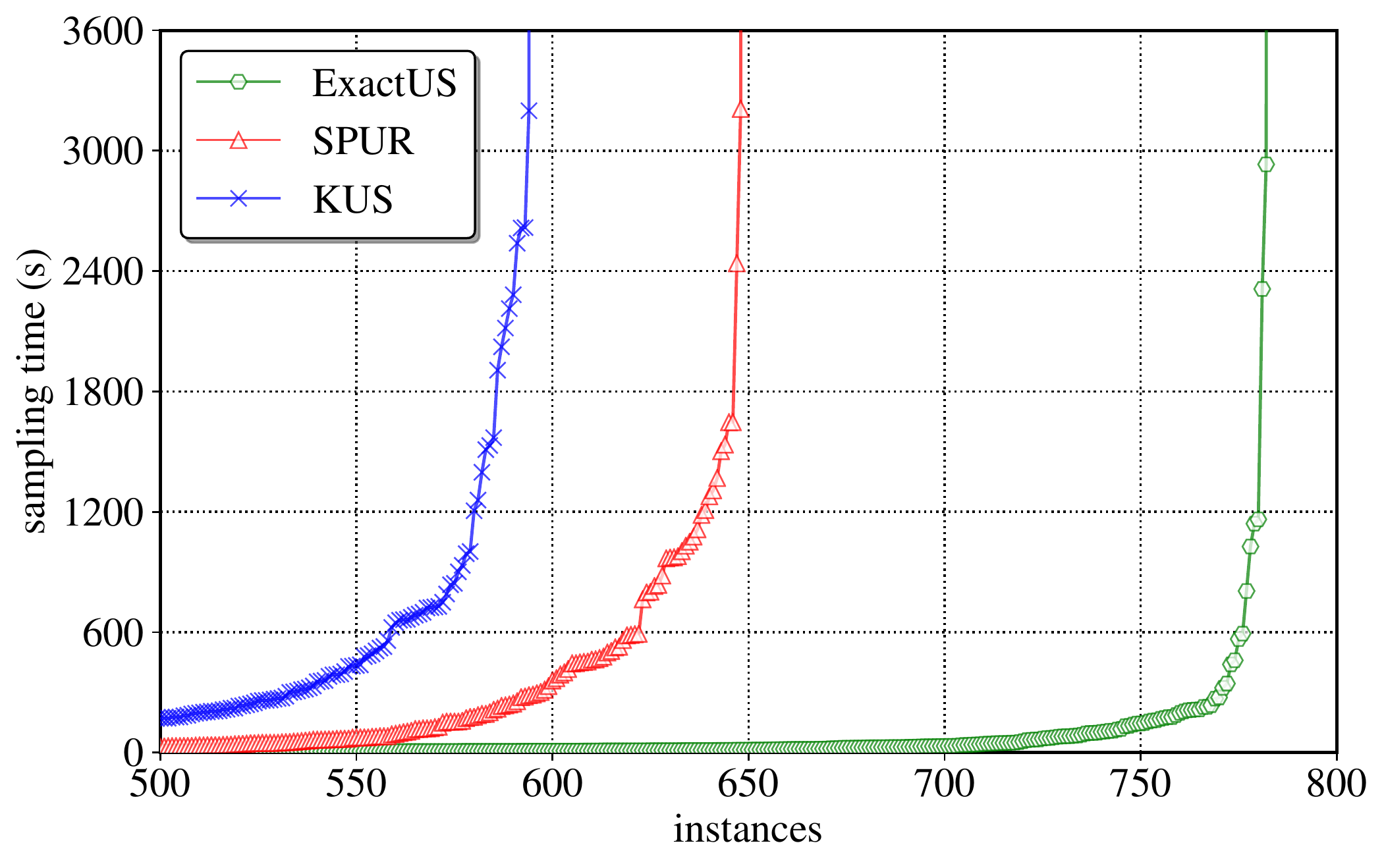}
	\caption{Cactus plot comparing the solving time of different samplers. (Best viewed in color)}\label{fig:sampling:time}
\end{figure}

\section{Discussion on Tractability of \Ctwo{}}
\label{sec:tract}

We highlight that our focus in this paper is primarily on
improving the scalability of model counters and uniform samplers. 
However, encouraged by the significant performance improvement by {\ExactMC}
over existing solvers as shown in our experimental results, 
we investigate further into the underlying language, {\Ctwo}. 
To this end, we study {\Ctwo} from a knowledge compilation perspective characterize the tractability of {\Ctwo}. 
We refer the reader to Darwiche and Marquis's seminal work ~\cite{Darwiche:Marquis:02} for definitions of different standard operations in the literature. We focus on the five queries: implicant check, model counting, consistency check, validity check, and model enumeration. 

We first show that  \Ctwo{} supports tractable implicant check:

\begin{proposition}\label{prop:IM}
	Given a consistent term $T$ and a CCDD node $u$, we use $\mathit{IM}(T, u)$ to denote whether $T \models \vartheta(u)$. Then $\mathit{IM}(T, u)$ can be recursively performed in linear time:
	\begin{equation*}
% duplicate label
%\label{eq:sharp-DecDNNF}
	\mathit{IM}(T, u) = \begin{cases}
	\mathit{false} & sym(u) = \bot \\
	\mathit{true} & sym(u) = \top \\
	\mathit{IM}(T, lo(u)) & \text{$\NOT sym(u) \in T$} \\
	\mathit{IM}(T, hi(u)) & \text{$sym(u) \in T$} \\
	\bigwedge_{v \in Ch(u)}\mathit{IM}(T, v) & {\text{otherwise}}
	\end{cases}
	\end{equation*}
\end{proposition}
\begin{proof}
	The constant, and decomposed  and kernelized conjunction cases are obvious, and thus we focus on the decision case. Note that a literal equivalence is a special decision node. For the case where $\neg sym(u) \in T$, each model of $T$ is not a model of $sym(u) \land \vartheta(hi(u))$, and thus $T \models \vartheta(u)$ iff $T \models \vartheta(lo(u))$. The case where $sym(u) \in T$ is similar. Otherwise, $T \models \vartheta(u)$ iff $\neg sym(u) \land T \models \neg sym(u) \land \vartheta(lo(u))$ and $sym(u) \land T \models sym(u) \land \vartheta(hi(u))$ iff $T \models \vartheta(lo(u))$ and $T \models \vartheta(hi(u))$. 
\end{proof}

 Since {\Ctwo} supports model counting in linear time, we obtain that {\Ctwo} supports consistency check, validity check, and model enumeration in polynomial time. 
\begin{theorem}\label{thm:succinct}
% \Ctwo{} satisfies \CT{}, \CO{}, \VA{}, \ME{} and {\IM}. 
\Ctwo{} supports model counting, consistency check, validity check, and implicant check in time polynomial in the DAG size, and supports model enumeration in time polynomial in both the DAG size and model count.
\end{theorem}

According to the notation in the knowledge compilation map \cite{Darwiche:Marquis:02}, we know that \Ctwo{} satisfies \CT{}, \CO{}, \VA{}, {\IM}, and \ME{}, respectively. We mention that if we restrict the number of $\AND_k$-nodes in each path from the root to a leaf, to be a constant $t$, we can obtain a subset of \Ctwo{}. This subset is still a superset of \DecDNNF{}, and supports the same tractable operations as \DecDNNF{}.
We remark that another representation in the knowledge compilation literature called \ensuremath{\mathsf{EADT}} \cite{Koriche:etal:13} uses a generalization of literal equivalence; however, \ensuremath{\mathsf{EADT}} is a tree-structured representation and therefore is not a generalization of \DecDNNF{}, which is a DAG-based representation.

\section{Conclusion}\label{sec:conclusion}

This paper proposed the notion of kernelization to capture literal equivalence in knowledge compilation. 
%This property can be orthogonal to many other properties in the KC map. 
Combining kernelization, decomposition and ordered decision, this paper identified the new language \Ctwo{}.
\Ctwo{} supports two key queries, model counting and uniform sampling in polynomial time. 
We designed tractable algorithms for model counting and uniform sampling on \Ctwo{}.
To facilitate the usage of \Ctwo{} in practice, we developed the prototype compiler \Panini{} to compile \CNF{} formulas into \Ctwo{}. 
Experimental results show that our compilation times are better with smaller representations than state-of-art \DecDNNF{}, \LSDD{}, and \LOBDDC{}{} compilers.
For model counting and uniform sampling, our techniques also significantly outperform the state-of-the-art tools. 
Since kernelization is orthogonal to other notions such as determinism and decomposability, we expect kernelization  will help the knowledge compilation community to identify more interesting languages.

\section*{Acknowledgments}
We are grateful to the anonymous reviewers for their constructive feedback.
We thank Mate Soos and Arijit Shaw for their help. 
This work was supported in part by the National Research Foundation Singapore under its NRF Fellowship Programme [NRF-NRFFAI1-2019-0004] and the AI Singapore Programme [AISG-RP-2018-005], NUS ODPRT9 Grant [R-252-000-685-13], Jilin Province Natural Science Foundation [20190103005JH] and National Natural Science Foundation of China [61806050]. 
The computational resources were provided by
the National Supercomputing Centre, Singapore (\url{https://www.nscc.sg}).



%%       \citet{<label>} ==> Jones et al. [21]
%%       \citep{<label>} ==> [21]
%%


\bibliographystyle{elsarticle-num-names} 
\bibliography{papers,software,books,sigproc}




%%\bibitem[ ()]{}

%%\end{thebibliography}
\end{document}